%% file: main.tex
\documentclass[letterpaper]{article} 
\usepackage{aaai2026}  
\usepackage{times}  
\usepackage{helvet}  
\usepackage{courier}  
\usepackage{color}
\usepackage[table]{xcolor}
\usepackage[hyphens]{url}  
\usepackage{graphicx} 
\usepackage{natbib}  
\usepackage{colortbl}
\usepackage{caption} 
\usepackage{multirow}
\usepackage{booktabs}
\usepackage{xcolor}
\usepackage{amsmath} 
\usepackage{amsfonts} 
\usepackage{algorithm}
\usepackage{algorithmic}

\usepackage{amsthm}      
\usepackage{cleveref}    
\newtheorem{theorem}{Theorem}

\definecolor{mygray}{gray}{.6}
\definecolor{myblue}{RGB}{89,158,254}
\definecolor{mygreen1}{RGB}{81,150,111}
\definecolor{mygreen2}{RGB}{93,174,86}
\definecolor{myred}{RGB}{160,0,0}
\definecolor{myyellow}{RGB}{227,207,87}
\usepackage{newfloat}
\usepackage{listings}
\DeclareCaptionStyle{ruled}{labelfont=normalfont,labelsep=colon,strut=off} 
\floatstyle{ruled}
\newfloat{listing}{tb}{lst}{}
\floatname{listing}{Listing}
\title{FQ-PETR: Fully Quantized Position Embedding Transformation \\ for Multi-View 3D Object Detection}
\author{
    Jiangyong Yu\textsuperscript{\rm 1},
    Changyong Shu\textsuperscript{\rm 1}\thanks{Corresponding author.},
    Sifan Zhou\textsuperscript{\rm 1,2}\thanks{This work was done during an internship at HOUMO AI.},
    Zichen Yu\textsuperscript{\rm 3},
    Xing Hu\textsuperscript{\rm 1},
    Yan Chen\textsuperscript{\rm 1},
    Dawei Yang\textsuperscript{\rm 1}\footnotemark[1]\\
}
\affiliations{
    \textsuperscript{\rm 1}Houmo AI \\
    \textsuperscript{\rm 2}Southeast University \\
    \textsuperscript{\rm 3}Dalian University of Technology\\
    \{jiangyong.yu, changyong.shu, xing.hu, yan.chen, dawei.yang\}@houmo.ai,\\
    sifanjay@gmail.com,
    yuzichen@mail.dlut.edu.cn
}

\usepackage{bibentry}

\begin{document}

\maketitle

\begin{abstract}
Camera-based multi-view 3D detection is crucial for autonomous driving. PETR and its variants (PETRs) excel in benchmarks but face deployment challenges due to high computational cost and memory footprint. Quantization is an effective technique for compressing deep neural networks by reducing the bit width of weights and activations. However, directly applying existing quantization methods to PETRs leads to severe accuracy degradation. This issue primarily arises from two key challenges: (1) significant magnitude disparity between multi-modal features—specifically, image features and camera-ray positional embeddings (PE), and (2) the inefficiency and approximation error of quantizing non-linear operators, which commonly rely on hardware-unfriendly computations. In this paper, we propose \textbf{FQ-PETR}, a fully quantized framework for PETRs, featuring three key innovations: (1) Quantization-Friendly LiDAR-ray Position Embedding (QFPE): Replacing multi-point sampling with LiDAR-prior-guided single-point sampling and anchor-based embedding eliminates problematic non-linearities (e.g., inverse-sigmoid) and aligns PE scale with image features, preserving accuracy. (2) Dual-Lookup Table (DULUT): This algorithm approximates complex non-linear functions using two cascaded linear LUTs, achieving high fidelity with minimal entries and no specialized hardware. (3) Quantization After Numerical Stabilization (QANS): Performing quantization after softmax numerical stabilization mitigates attention distortion from large inputs. On PETRs (e.g. PETR, StreamPETR, PETRv2, MV2d), FQ-PETR under W8A8 achieves near-floating-point accuracy ($<$ 1\% degradation) while reducing latency by up to 75\%, significantly outperforming existing PTQ and QAT baselines.
\end{abstract}

\begin{links}
    \link{Code}{https://github.com/JiangYongYu1/FQ-PETR}
\end{links}

\section{Introduction}
\begin{figure}[!tb]
\vspace{-3mm}
\centering
	\includegraphics[width=1.0\linewidth]{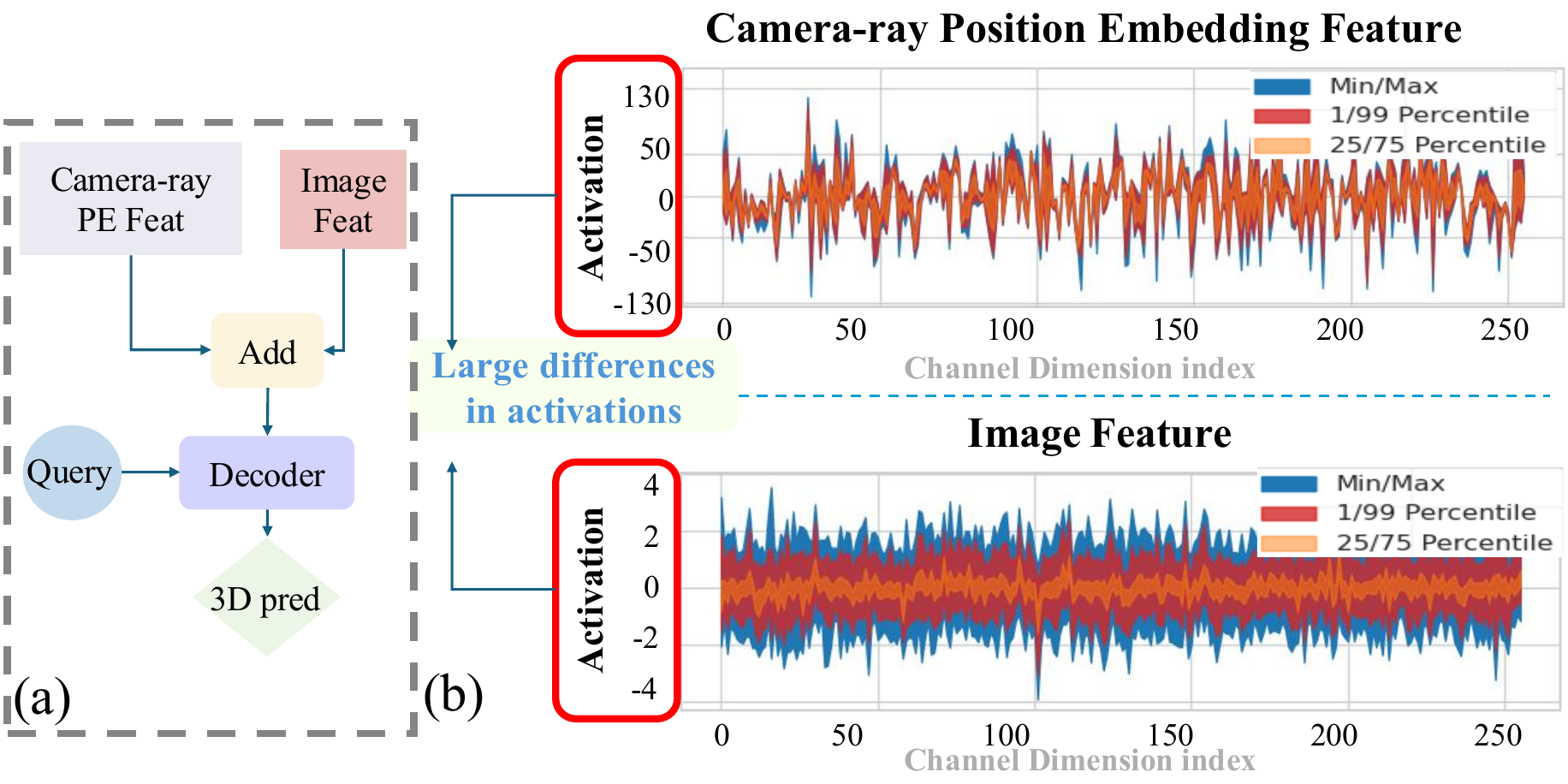}
    \vspace{-0.6cm}
	\caption{(a) Illustration of PETR pipeline. The camera-ray position embedding and image features are fused via element-wise addition, and then combined with object queries through a Transformer decoder to generate 3D object predictions. (b) The activation values of camera-ray position embedding feature range form -130 to 130, where image feature are primarily centered around 0.}
	\label{fig:figure1}
\vspace{-6mm}
\end{figure}

\begin{figure}[ht]
\centering
	\includegraphics[width=1.0\linewidth]{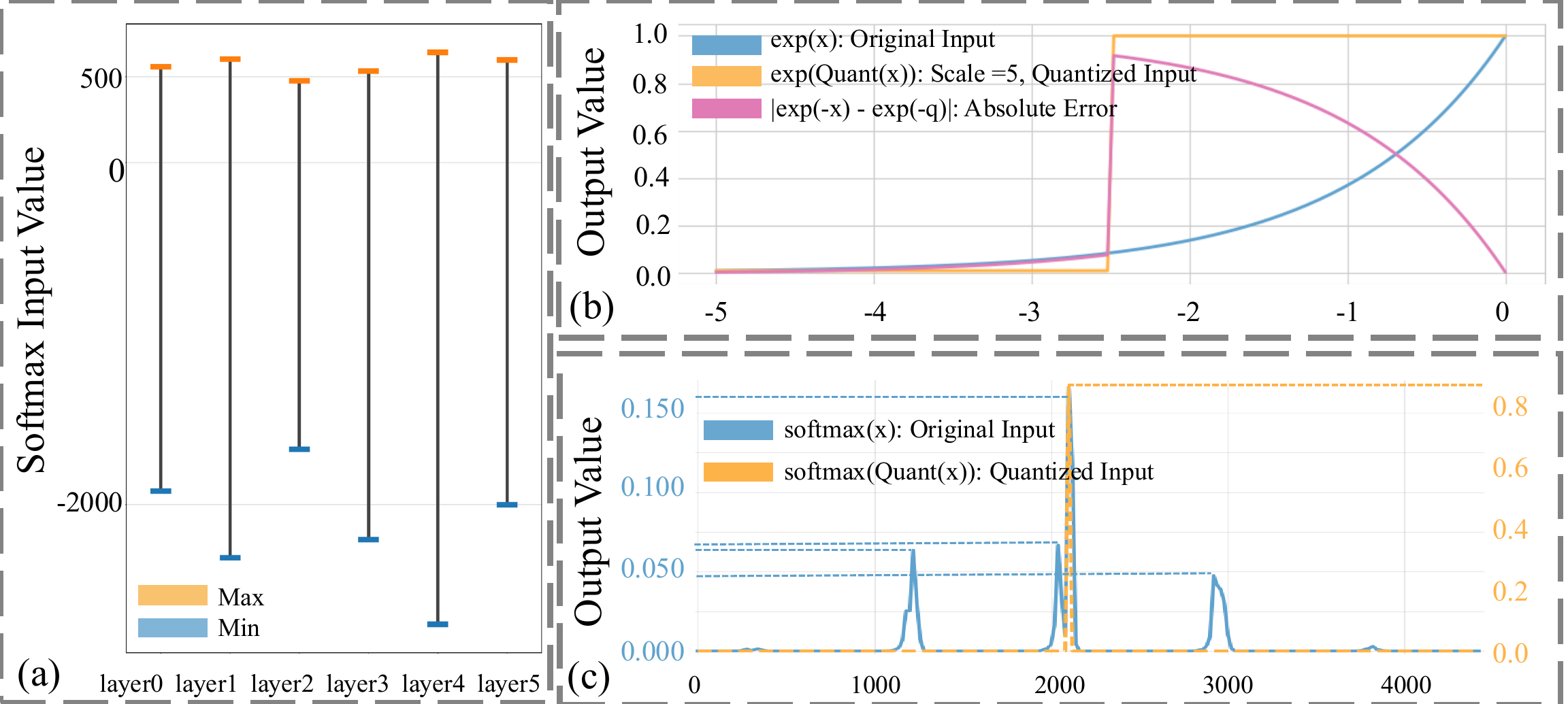}
    \vspace{-0.4cm}
    \caption{Softmax input range and quantization effects. (a) Logits in PETR’s cross-attention span an extremely wide interval, forcing a large per-tensor INT8 scale.
    (b) With a representative scale \(s=5\), rounding the inputs to \(\exp(\cdot)\) introduces pronounced piece-wise errors \(\lvert\exp(x)-\exp(\hat{x})\rvert\).
    (c) Softmax outputs from quantized inputs (orange) exhibit peak attenuation and positional shift compared with the floating-point baseline (blue), indicating attention distortion.}
	\label{fig:figure2}
\vspace{-5mm}
\end{figure}

Camera-based multi-view 3D object detection has emerged as a pivotal technology for autonomous driving systems, offering unparalleled cost-effectiveness while providing rich semantic information critical for scene understanding. This approach has gained traction over LiDAR-based solutions due to its lower hardware cost, higher resolution for distant objects, and seamless integration with existing vehicle camera systems. Among various paradigms for multi-camera 3D perception, PETR and its variants (collectively PETRs)~\cite{liu2022petr,liu2022petrv2,streampetr} have established new state-of-the-art benchmarks by ingeniously adapting the transformer-based DETR framework to 3D space. The core innovation lies in their position-aware feature representation: camera-ray positional embeddings transform 2D image features into 3D space coordinates, enabling direct interaction between object queries and 3D features via transformer decoders (Fig.~\ref{fig:figure1}a). This paradigm shift has propelled PETRs to the forefront of major autonomous driving benchmarks like nuScenes and Waymo Open Dataset. Despite their impressive accuracy, PETRs present critical deployment challenges for automotive systems: their substantial computational demands and memory footprint create significant barriers to achieving real-time inference under strict latency and power constraints—essential requirements for safety-critical autonomous driving applications.

Quantization~\cite{yang2024post} is an effective method for compressing deep neural networks by reducing the bit width of weights and activations, resulting in lower memory usage, reduced computational complexity, and faster inference. However, we identify that PETRs exhibit unique quantization challenges distinct from other vision models. \textbf{First, the modality disparity in feature fusion.} As revealed in Fig.~\ref{fig:figure1}b, camera-ray positional embeddings exhibit a dynamic range ($\pm$130) nearly two orders of magnitude larger than image features ($\pm$4). During feature fusion (element-wise addition), this disparity forces quantization scales to be dominated by positional embeddings, compressing image features into merely 3--5 effective integer bins. Consequently, over 90\% of image feature information is discarded, leading to catastrophic accuracy degradation---up to 20.8\% mAP drop in our experiments. \textbf{Second, non-linear operators present dual challenges.} 
\textit{Challenge 1: Attention distortion from extreme-value distributions.} 
Inputs to operators like Softmax exhibit wide dynamic ranges in PETR (Fig.~\ref{fig:figure2}a). As shown in Fig.~\ref{fig:figure2}b, directly quantizing the inputs to these operators yields large errors. For Softmax, this further causes a severe attention shift (Fig.~\ref{fig:figure2}c).
\textit{Challenge 2: Hardware implementation overhead.} 
The computational burden manifests differently across platforms: 
(1) On GPUs, specialized units (e.g., NVIDIA Tensor Cores) show 3--5$\times$ lower throughput for Softmax versus GEMM~\cite{dao2023flashattention}; 
(2) On edge devices, lookup tables (LUTs)~\cite{wang2018look} face memory-complexity tradeoffs---linear implementations require exponential memory growth while non-uniform alternatives~\cite{yu2022nn} demand complex hardware; 
(3) Software approximations~\cite{kim2021bert,li2023vit} sacrifice precision for deployability.

The co-existence of these bottlenecks---extreme modality disparity and non-linear computation overhead---creates a unique challenge for PETRs quantization that remains unaddressed by existing approaches. To bridge this gap, we propose FQ-PETR---the first framework enabling fully integer inference for PETRs without accuracy compromise. Our co-design approach integrates three synergistic innovations targeting the identified bottlenecks:

(1) \textbf{Quantization-Friendly LiDAR-ray PE (QFPE)} rethinks position embedding from first principles. Inspired by LiDAR physics, we sample a single 3D point per pixel along depth rays, eliminating multi-point interpolation and inverse-sigmoid distortion (Fig.~\ref{fig:distribution_before_and_after_insigmoid}b). Anchor-based embeddings with learnable bounds constrain dynamic range while enhancing spatial coherence. This hardware-aligned design reduces position magnitude by 4.4$\times$ while improving detection accuracy.

(2) \textbf{Dual-Lookup Table (DULUT)} decouples non-uniform approximation from hardware dependencies. By cascading two linear LUTs---the first acting as a ``nonlinear index mapper''---our algorithm achieves exponential entry reduction (e.g., 32+32 entries for SiLU) while maintaining $<$0.1\% approximation error. Crucially, DULUT requires only standard LUT units, making it deployable across diverse edge platforms.

(3) \textbf{Quantization After Numerical Stabilization (QANS)} specifically targets Softmax quantization. By applying integer conversion after logit subtraction (numerical stabilization), we constrain inputs to a non-positive range [$-\beta$, 0]. Adaptive $\beta$ selection minimizes distribution shift, preserving attention patterns with 8-bit precision.

Comprehensive evaluations demonstrate FQ-PETR's versatility across PETR and its variants (e.g. PETR, StreamPETR, PETRv2, and MV2d). Under full integer quantization (W8A8), it achieves near-floating-point accuracy ($<$1\% mAP/NDS degradation) with 75\% latency reduction and 4$\times$ model compression---significantly outperforming state-of-the-art PTQ/QAT baselines. These advancements provide a critical pathway toward deploying high-performance 3D perception in resource-constrained vehicles.

\section{Related Work}

\paragraph{Multi-View 3D Object Detection.} The RGB camera has various applications in computer vision~\cite{zhang2023multi,tao2023dudb,wang2024scantd,wang2025target,zhao2023benchmark,zhao2024balf,dai2025unbiased,yin2025knowledge}. Surround-view 3D object detection is essential for autonomous driving and is generally categorized into LSS-based~\cite{bevdet,bevdepth,bevdet4d} and transformer-based~\cite{liu2022petr,shu20233DPPE} approaches. LSS-based methods project multi-camera features onto dense BEV (Bird's Eye View) representations~\cite{lss}, but their high memory consumption hinders efficient long-range perception. Transformer-based methods leverage sparsity to enhance long-distance perception. Among these, the PETR series have gained significant attention. PETR~\cite{liu2022petr} transforms 2D image features into 3D representations using 3D positional embedding. PETRv2~\cite{liu2022petrv2} introduces temporal feature indexing, while StreamPETR~\cite{streampetr} extends temporal query processing. Some works~\cite{wang2022focal, wang2023object, chu2024rayformer} accelerate processing by incorporating 2D detection priors. CMT~\cite{cmt} fuses vision and LiDAR point clouds. Improvements to PETR's positional embedding have also been explored~\cite{shu20233DPPE, hou2024open}. Additionally, PETR has been fused into the Omnidrive~\cite{wang2024omnidrive} to enhance 3D perception with large models.

\paragraph{Model Quantization.}
Quantization compresses models by converting weights and activations from floating-point to lower-bit integer representations~\cite{LQNets_eccv2018,PTQ_4bit_rapid_deployment_nips2019,Low_bit_quant_iccvw2019}. Among various methods~\cite{wei2022outlier,ptq4ris,liu2024spinquant,ashkboos2024slicegpt,lu2024generic,li2025sepprune,mambaquant,moequant,xuchen2025}, we focus on uniform symmetric quantization, mapping floating-point values \( x_f \) to discrete \( k \)-bit integer values \( x_q \) as:
\begin{equation}\label{e:eq1}
\begin{aligned}
x_q = \text{clamp}\left( \left\lfloor \frac{x_f}{s} \right\rceil, -2^{k-1}, 2^{k-1}-1 \right),
\end{aligned}
\end{equation}
where \( s \) is the scaling factor computed as:
\begin{equation}\label{e:eq2}
\begin{aligned}
s = \frac{x_f^{\text{max}} - x_f^{\text{min}}}{2^k},
\end{aligned}
\end{equation}
with \( x_f^{\text{max}} \) and \( x_f^{\text{min}} \) being the maximum and minimum floating-point values from the calibration dataset. Quantization methods are categorized into Quantization-Aware Training (QAT) and Post-Training Quantization (PTQ). QAT~\cite{esser2019learned,bhalgat2020lsq+} introduces quantization-aware losses during training, enhancing robustness but requiring resource-intensive retraining. Compared to QAT, PTQ offers rapid deployment without retraining. While PTQ methods have been successful on CNNs~\cite{nagel2019data,nagel2020up,li2021brecq}, they often perform poorly on transformer-based 3D detectors due to structural differences. For ViTs, practical PTQ algorithms have been developed~\cite{yuan2022ptq4vit,lin2021fq,li2023repq,shi2024p}. In the context of transformer-based object detection models, Q-DETR~\cite{xu2023q} and AQ-DETR~\cite{wang2024aq} use QAT and knowledge distillation to mitigate performance degradation in low-bit quantization of DETR models. SmoothQuant~\cite{xiao2023smoothquant} and OmniQuant~\cite{shao2023omniquant} shift quantization difficulty from activations to weights through
mathematical transformation. Recently, QuaRot~\cite{ashkboos2024quarot} uses random rotation matrices to enable
4-bit quantization of weights and activations, while OstQuant~\cite{hu2025ostquant} learns these matrices
to refine 4-bit quantization. MQuant~\cite{yu2025mquant} extends the rotate solution to multimodal language models. These methods primarily focus on quantizing GEMM operations. For nonlinear activation functions, lookup table (LUT) techniques~\cite{wang2018look} are commonly used. Additionally, methods like I-BERT~\cite{kim2021bert}, I-ViT~\cite{li2023vit} and I-LLM~\cite{hu2024llm} employ integer approximation to achieve fixed-point computation.

\noindent\textbf{Quantization for 3D Object Detection.}
Quantization methods have been applied to accelerate 3D object detection in autonomous driving and robotics. Leveraging advances in image quantization, QD-BEV~\cite{zhang2023qd} employs QAT and distillation in multi-camera 3D detection, achieving smaller models and faster inference than the \textit{BEVFormer} baseline~\cite{li2022bevformer}. For LiDAR-based detection, LIDAR-PTQ~\cite{zhou2024lidar} achieves state-of-the-art quantization on different 3D detector, with performance close to FP32 and $\sim3$$\times$ speedup. PillarHist~\cite{pillarhist,zhou2024information} improves pillar encoders via entropy-guided height histogram encoding, enhancing the quantization ability for various lidar-based 3D detector~\cite{centerpoint,zhou2023fastpillars}. Point4bit~\cite{point4bit} propose a unified 4-bit PTQ framework for efficient low-bit deployment with minimal accuracy loss to various 3D perception tasks under real-world hardware constraints. To our knowledge, there are no PTQ solution tailored for transformer-based 3D detection in autonomous driving.

\section{Methodology}
\label{sec:3}
In this section, we propose a fully quantization framework called \textbf{FQ-PETR} specifically designed for PETR series models. First, we present a Quantization-Friendly LiDAR-ray position embedding (QFPE) to reduce the gap between the numerical range of 3D PE feature and image feature while maintaining floating-point accuracy. Second, we introduce a dual-lookup table (DULUT), which abstracts the index comparator required for nonlinear lookup tables into a nonlinear function, and uses a linear LUT to equal it. It maintains high approximation fidelity with fewer table entries without requiring special hardware support. Third, we propose quantization after numerical stabilization (QANS), which solves the problem of attention shift caused by quantization when the input range of Softmax is too large.

\begin{figure}[!htbp]
\vspace{-2mm}
	\includegraphics[width=1\linewidth]{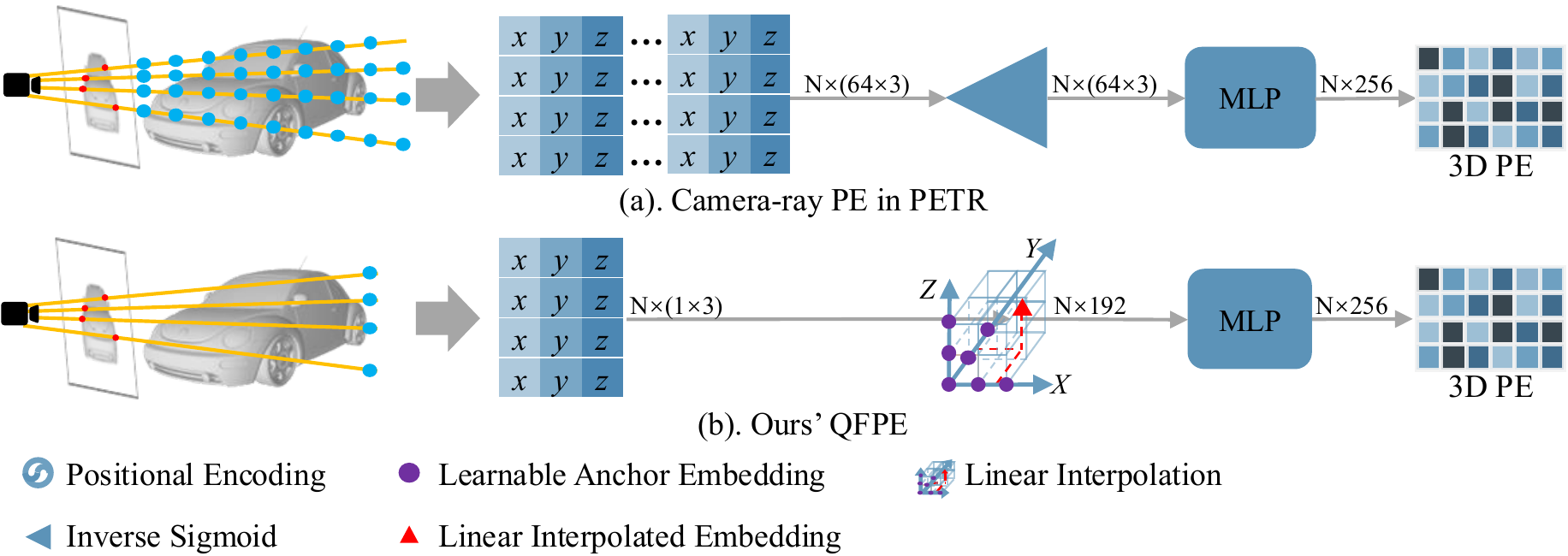}
	\caption{
 The overall architecture comparison of camera-ray PE, lidar-ray PE and our QFPE.
	}
    \vspace{-2mm}
	\label{fig:pe_compare}
    \vspace{-3mm}
\end{figure}

\subsection{Quantization-Friendly LiDAR-ray Position Embedding (QFPE)}\label{sec:3.1}
We first analyze the quantization failure from the perspective of magnitude imbalance, demonstrating that structural optimization of the model is essential for resolution. Based on this analysis, we present our QFPE design solution.

\paragraph{Magnitude Imbalance in PETR}
\label{subsec:imbalance} As evidenced in Fig.~\ref{fig:figure1}(b), a significant modality discrepancy exists: the camera-ray position embedding (Camera-ray PE) exhibits a dynamic range of approximately ±130, while multi-view image features are confined within ±4.

Crucially, when quantizing the element-wise summation of the above two features, regardless of the quantization method employed, the scale factor becomes dominated by the Camera-ray PE. This leads to severe compression of image features and consequent catastrophic accuracy degradation. Therefore, addressing this quantization challenge necessitates architectural modifications to the model.

\begin{figure}[htb]
\vspace{-3mm}
\centering
	\includegraphics[width=1.0\linewidth]{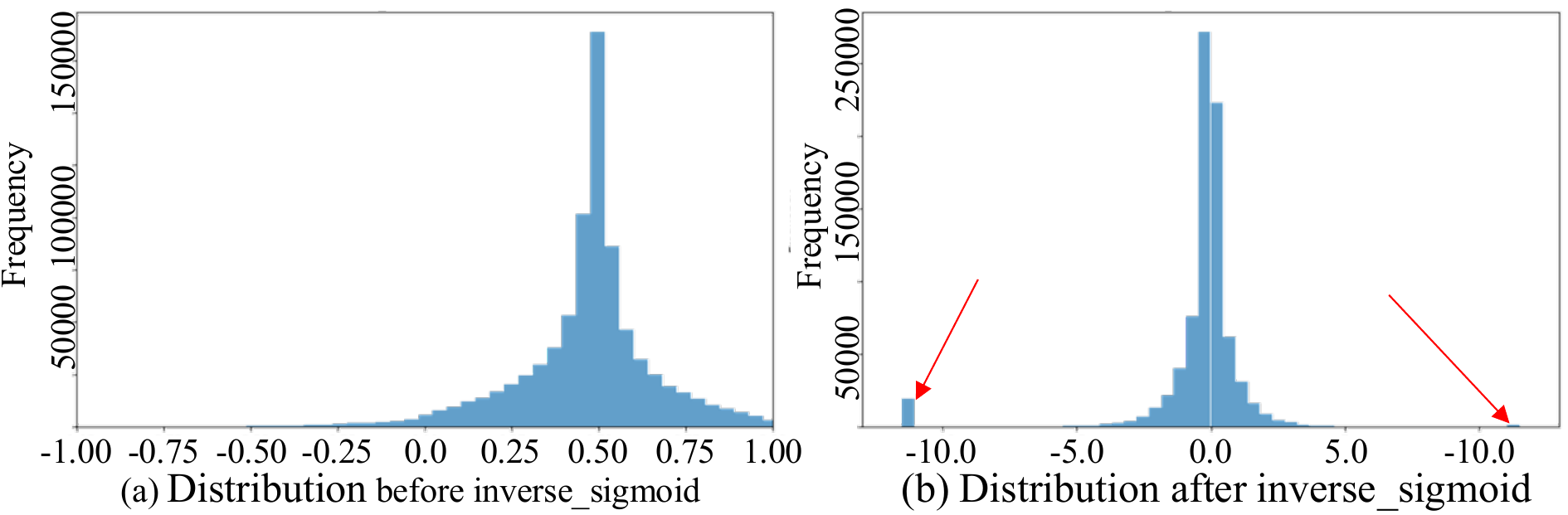}
    \vspace{-0.6cm}
	\caption{Distribution before and after inverse-sigmoid.}
	\label{fig:distribution_before_and_after_insigmoid}
    \vspace{-3mm}
\end{figure}

Through systematic analysis of the Camera-Ray PE construction pipeline (shown in Fig.~\ref{fig:pe_compare}(a)) and Magnitude Propagation Analysis in Appendix A, we identify that the \textit{inverse-sigmoid} operation amplifies the magnitude by approximately 11.5×. Furthermore, this operation distorts the originally balanced depth distribution (Fig.~\ref{fig:distribution_before_and_after_insigmoid}(a)), skewing it toward extreme outliers and thereby exacerbating the Camera-ray PE magnitude inflation.

\paragraph{QFPE Architecture.} To simultaneously address both the magnitude imbalance and outliers challenges, we propose a \textbf{Q}uantization-\textbf{F}riendly LiDAR-ray \textbf{P}osition \textbf{E}mbedding (QFPE) that achieves significant magnitude reduction while maintaining computational efficiency. As illustrated in Fig.~\ref{fig:pe_compare} (b), our design incorporates two key innovations:

\begin{enumerate}
\item \textbf{LiDAR-prior Guided Single-point Sampling}. Leveraging the physical characteristics of LiDAR sensors, we implement a sparse sampling strategy that selects only one 3D point per pixel along each depth ray (Fig.~\ref{fig:pe_compare}(b)), contrasting with the multi-sample approach in conventional camera-ray PE. This design eliminates the need for iterative inverse-sigmoid transformations, substantially reducing embedding variance.

\item \textbf{Anchor-based Constrained Embedding with Convex Combination}. We establish three axis-aligned anchor embeddings $\{E_\alpha^i\}_{i=1}^3$ for each spatial dimension $\alpha \in \{x,y,z\}$, accompanied by corresponding anchor locations $\{L_\alpha^i\}_{i=1}^3$. For any LiDAR-sampled 3D point $(x_j,y_j,z_j)$, we compute axis-specific embeddings through linear interpolation between adjacent anchors:
\begin{equation}
\begin{aligned}
    e_\alpha^j &= \frac{p_\alpha - L_\alpha^{i_\alpha}}{L_\alpha^{i_\alpha+1} - L_\alpha^{i_\alpha}}\,E_\alpha^{i_\alpha+1} 
    + \frac{L_\alpha^{i_\alpha+1} - p_\alpha}{L_\alpha^{i_\alpha+1} - L_\alpha^{i_\alpha}}\,E_\alpha^{i_\alpha}
\end{aligned}
\end{equation}
where $p_\alpha$ denotes the coordinate along axis $\alpha$. The final positional embedding vector is generated by concatenating these axis-wise embeddings and processing them through a lightweight MLP.
\end{enumerate}

\begin{figure}[!t]
\centering
\includegraphics[width=0.9\linewidth]{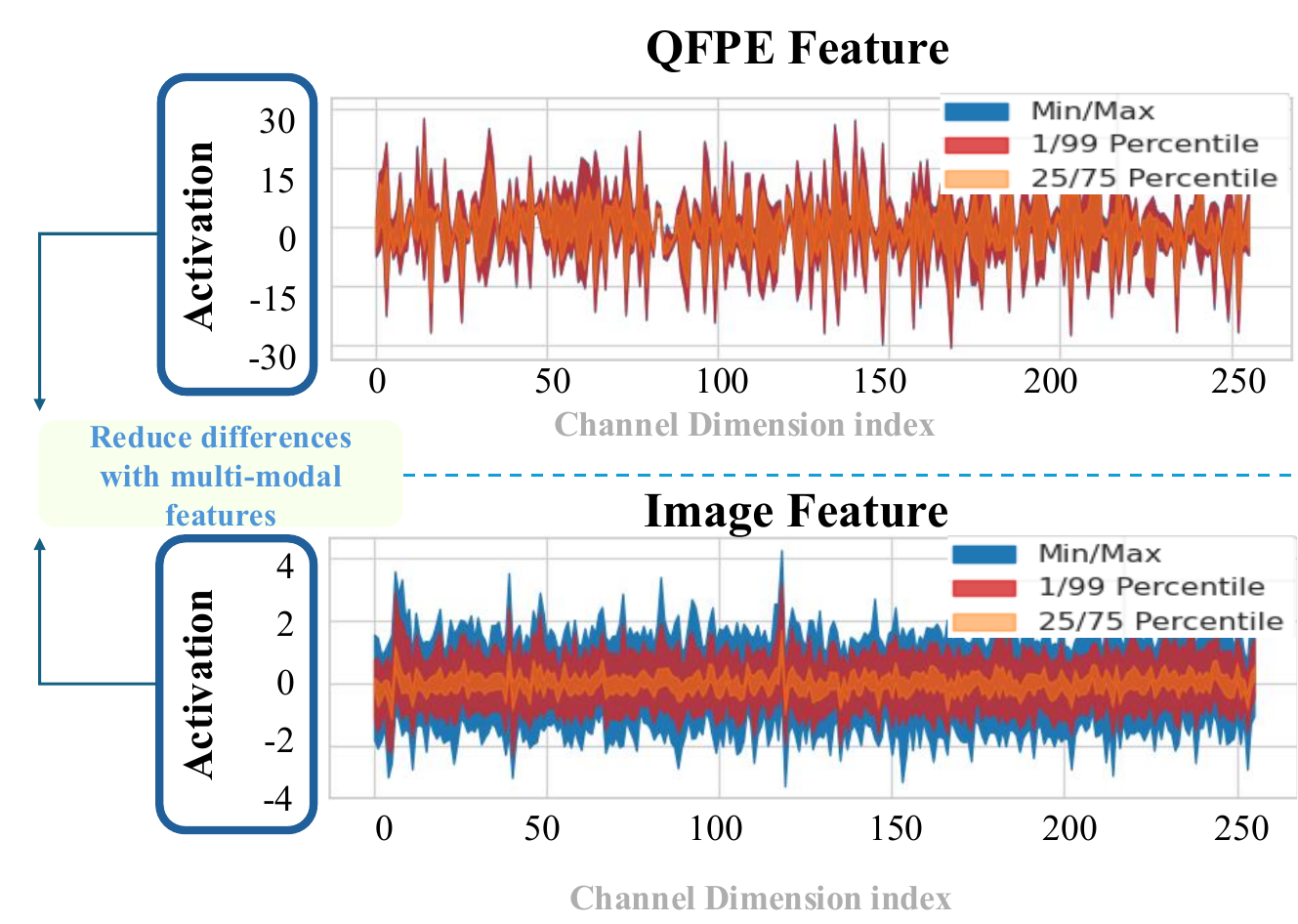}
    \vspace{-0.3cm}
	\caption{The activation values of QFPE feature range form -29.7 to 29.7.}
	\label{fig:QFPE_img_compare}
\vspace{-4mm}
\end{figure}

\begin{figure}[ht]
\centering
	\includegraphics[width=1\linewidth]{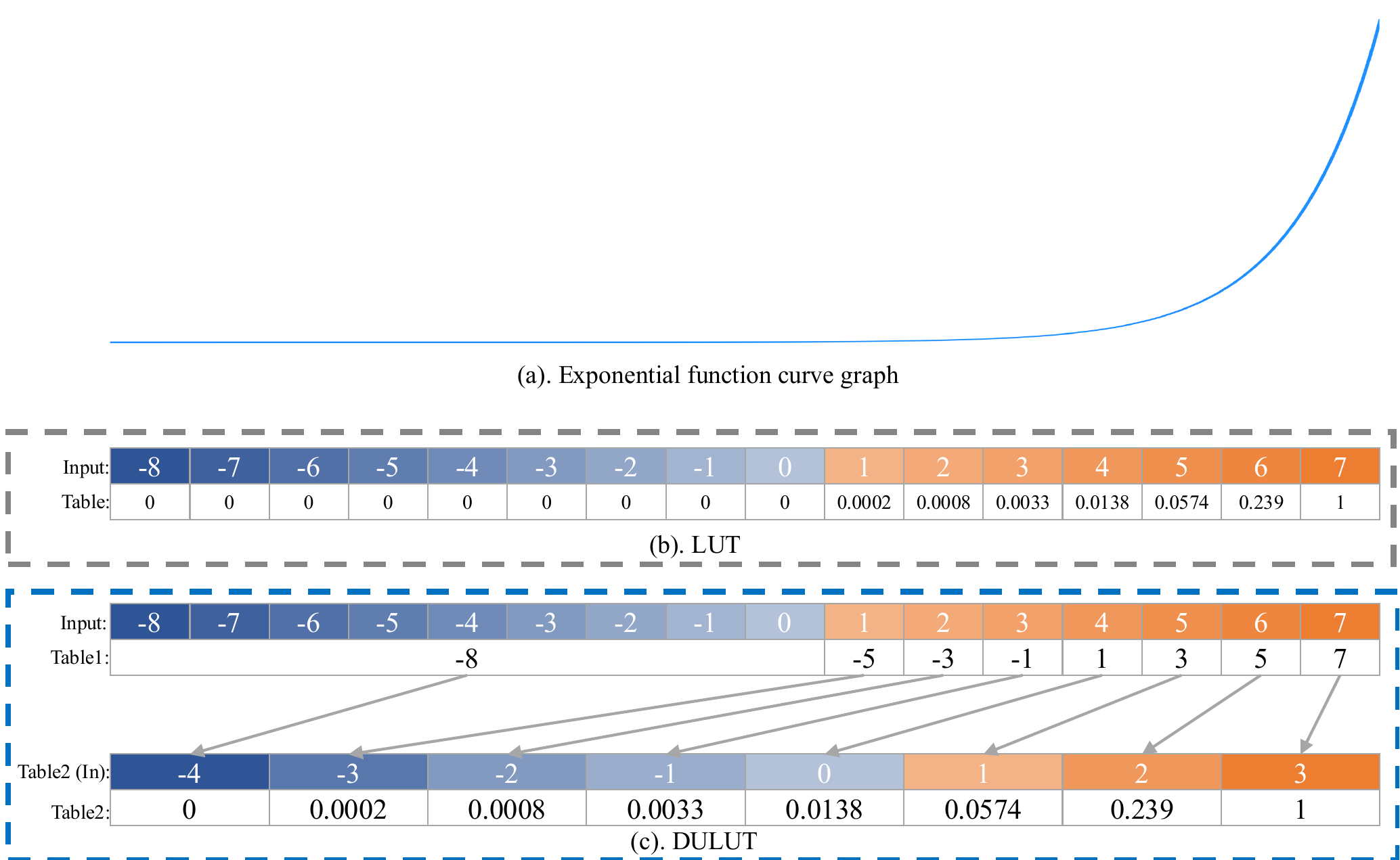}
    \vspace{-0.3cm}
	\caption{INT4 example for $\exp(\cdot)$ with DULUT. 
A 16-entry linear LUT is replaced by two 8-entry tables: 
(i) an index-mapping table that compresses near-flat input ranges, and 
(ii) a value table for interpolation. 
Precision matches the 16-entry LUT with fewer entries.}
	\label{fig:DULUT}
    \vspace{-0.5cm}
\end{figure}

These two innovations ensure our QFPE remains both bounded in magnitude and free of difficult-to-quantize nonlinearities. Fig.~\ref{fig:QFPE_img_compare} shows that the dynamic range of our QFPE ($\pm 29.7$) is only marginally wider than that of standard image features ($\pm 4$)—in stark contrast to PETR’s original ($\pm 127.3$). The proposed design eliminates complex nonlinear operations (inverse-sigmoid), achieving hardware-compatible computation without compromising geometric fidelity.

\subsection{DULUT for Non-linear Functions.} 
\paragraph{Limitations of linear-LUT and NN-LUT.} Linear-LUT requires exponentially more entries as input bit-width increases, quickly exhausting on-chip SRAM. Meanwhile, NN-LUT leverages neural networks to search non-uniform intervals, concentrating entries in high-curvature regions. However, NN-LUT heavily depends on training data distribution and random initialization, requires hardware-specific comparator trees, and only achieves real acceleration with dedicated hardware support. Moreover, the non-uniform intervals complicate fusion and parallel optimization in existing compiler frameworks.
\paragraph{Proposed DULUT Algorithm.}
To overcome these limitations, we first establish a theoretical error bound for linear interpolation. For a twice-differentiable function $f(x)$ interpolated linearly over an interval $[x_i, x_{i+1}]$, the maximum interpolation error is bounded by:
\begin{equation}
\resizebox{0.92\linewidth}{!}{%
$\displaystyle
\max_{x\in[x_i,x_{i+1}]} |f(x)-P(x)|
\le
\frac{(x_{i+1}-x_i)^2}{8}\,
\max_{x\in[x_i,x_{i+1}]}|f''(x)|
+ \varepsilon_{\mathrm{hw}},$
}
\label{eq:DULUT_error_bound}
\end{equation}
where $f''(x)$ denotes the second derivative (curvature), and $\varepsilon_{\mathrm{hw}}$ represents hardware-induced error from finite interpolation resolution and rounding. Hence, intervals must be shortened in regions with large curvature (``enlarge'' regions), while flatter or near-saturated areas with lower curvature can be merged into fewer intervals (``shrink'' regions).

Instead of designing hardware-based non-uniform indexing comparators, DULUT treats the nonlinear comparator itself as a nonlinear function $g(x)$ and approximates it with an additional linear LUT. Thus, the indexing operation reduces to a linear LUT lookup, and the overall problem simplifies to identifying the optimal nonlinear function $g(x)$. Ideally, without hardware error constraints, one could approximate $g(x)$ directly from curvature information. However, due to hardware limitations (e.g., fixed interpolation step size), using curvature to merge intervals can introduce significant approximation errors (more details see Appendix B).

To address this issue, we propose an iterative optimization algorithm accounting explicitly for hardware constraints. We initialize two cascaded linear LUTs equivalently to a single linear LUT. Then, we evaluate the average relative error (ARE) for each interval and iteratively merge intervals with minimal ARE while subdividing intervals with maximal ARE. LUT parameters are updated accordingly if this operation reduces the global error. This optimization continues until no further error reduction can be achieved. The complete procedure is described in Algorithm~\ref{alg:cascade_lut}.

Empirically, DULUT achieves precision comparable to much larger single-table LUT implementations while significantly reducing SRAM overhead. In Fig.~\ref{fig:DULUT}, we illustrate an INT4 case for $\exp(\cdot)$. 
A linear LUT requires 16 entries to store the outputs for all 16 input codes. 
DULUT uses two 8-entry linear tables instead. 
The first table acts as an index mapper: it merges near-flat input ranges where $\exp$ varies little and allocates more indices to high-curvature ranges. 
The second table stores the values and performs linear interpolation. 
This preserves the precision of the 16-entry linear LUT while using fewer total entries. 

To facilitate reproduction, we give the entire process of DULUT and Triton implementation in Algorithm~\ref{algo:algorithm_DULUT}.

\begin{algorithm}[!ht]
\caption{Iterative Optimization Algorithm for DULUT}
\small
\label{alg:cascade_lut}
\begin{algorithmic}[1]
\REQUIRE Target nonlinear function $f(x)$ (FP32); quantization bit‐width $b$; hardware interpolation resolution $k$ (integer points per segment); initial LUT sizes $m_1$, $m_2$; error tolerance $\delta$
\ENSURE Optimized integer LUTs $\text{table}_1$, $\text{table}_2$, and corresponding quantization parameters (scales and zero‐points)
\STATE \textbf{Initialization:} Uniformly partition integer domain $\mathcal{X}=[-2^{b-1}, 2^{b-1}-1]$ into $m_1$ intervals to construct initial $\text{table}_1$. Set $\text{table}_2[i]=f(i)$ and quantize according to hardware resolution $k$.
\REPEAT
    \STATE Compute approximation $\hat{f}(x)=\text{table}_2[\text{table}_1[x]]$ for all $x\in\mathcal{X}$ (or dense samples).
    \STATE Calculate average relative error (ARE) within each interval.
    \STATE Identify interval $j_{\min}$ with minimal ARE and interval $j_{\max}$ with maximal ARE.
    \STATE Merge interval $j_{\min}$ with its neighbor (\textit{shrink}), releasing one interval.
    \STATE Split interval $j_{\max}$ into two equal subintervals (\textit{enlarge}), consuming the released interval.
    \STATE Update and requantize both LUTs; recompute global maximum ARE.
\UNTIL{global maximum ARE $\le \delta$ or no further improvement achievable}
\end{algorithmic}
\end{algorithm}

\begin{algorithm}[ht]
\caption{Pseudo-code of DULUT with Triton}
\label{algo:algorithm_DULUT}
\begin{algorithmic}[1]
\REQUIRE Quantized input $i_q$ ($i_{\text{bit}}$-bit signed); pre-computed tables $\mathbf{T}_1$ ($2^{t_{\text{bit1}}}+1$ slots) and $\mathbf{T}_2$ ($2^{t_{\text{bit2}}}+1$ slots)
\ENSURE Quantized output $o_q$
\vspace{0.2em}
\STATE \textbf{Function} {\small \textsc{LUTKernel}}($i_q$, $\text{table}$, $t_{\text{bit}}$, $i_{\text{bit}}$)
\STATE \hspace{1em} $shift_{\text{bit}} \leftarrow i_{\text{bit}} - t_{\text{bit}}$; \quad $shift_{\text{num}} \leftarrow 1 \ll shift_{\text{bit}}$
\STATE \hspace{1em} $q_{\min}, q_{\max} \leftarrow -(1 \ll (i_{\text{bit}}-1)), (1 \ll (i_{\text{bit}}-1)) - 1$
\FOR{each index $offset$ \textbf{in parallel}}
    \STATE $x \leftarrow i_q[offset] + (1 \ll (i_{\text{bit}}-1))$ \COMMENT{signed $\to$ unsigned}
    \STATE $idx \leftarrow x \gg shift_{\text{bit}}$; \quad $p \leftarrow x \bmod shift_{\text{num}}$
    \STATE $t_l \leftarrow \text{table}[idx]$; \quad $t_r \leftarrow \text{table}[idx+1]$
    \STATE $S \leftarrow (shift_{\text{num}} - p)t_l + p t_r$
    \STATE $S \leftarrow (S + (1 \ll (shift_{\text{bit}}-1))) \gg shift_{\text{bit}}$
    \STATE $o_q[offset] \leftarrow \text{clip}(S, q_{\min}, q_{\max})$
\ENDFOR
\STATE \textbf{return} $o_q$
\vspace{0.3em}
\STATE \textbf{Offline table construction:} Use Algorithm~1.
\vspace{0.2em}
\STATE \textbf{Online inference:}
\STATE \hspace{1em} $idx \leftarrow$ \textsc{LUTKernel}($i_q$, $\mathbf{T}_1$, $t_{\text{bit1}}$, $i_{\text{bit}}$)
\STATE \hspace{1em} $o_q \leftarrow$ \textsc{LUTKernel}($idx$, $\mathbf{T}_2$, $t_{\text{bit2}}$, $i_{\text{bit}}$)
\STATE \textbf{return} $o_q$
\end{algorithmic}
\end{algorithm}

\begin{figure}[ht]
\vspace{-3mm}
\centering
\includegraphics[width=1\linewidth]{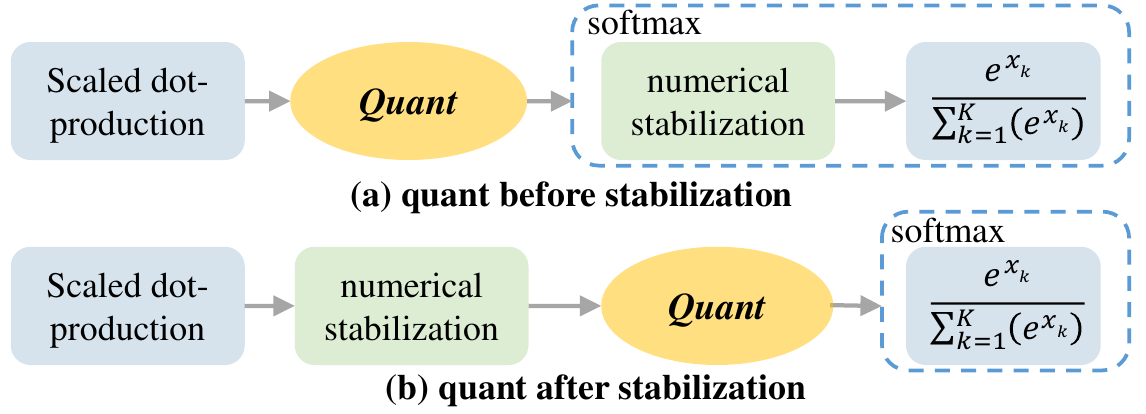}
    \vspace{-4mm}
	\caption{Illustration for quant before/after stabilization.}
	\label{fig:illustration_for_quant_before_after_stabilization}
    \vspace{-4mm}
\end{figure}

\subsection{Quantization After Numerical Stabilization} 
To solves the problem of attention shift caused
by quantization when the input range of Softmax is too large (see Figure\ref{fig:figure2}, we propose quantization After Numerical Stabilization (QANS) (Fig.~\ref{fig:illustration_for_quant_before_after_stabilization}), and adaptively determining the optimal truncation lower bound to minimize softmax error.

After NS, inputs for softmax are non-positive. Values below $-20$ approach zero after exponentiation, so we define a candidate set of scaling factors $S = {s_1, s_2, ..., s_N}$ with $s_i = \frac{i}{2^{k-1}}$ for $k$-bit quantization. The dequantized input is:
\begin{equation}
\vspace{-1mm}
\begin{aligned}
\hat{x}_s^i = s_i \cdot \text{clamp}\left( \text{round}\left( \frac{x_s}{s_i} \right), -2^{k-1}, 2^{k-1}-1 \right)
\end{aligned}
\end{equation}

ensuring $\hat{x}_s^i \in [-i, 0]$. We compute the softmax distributions $p_f = \text{softmax}(x_s)$ and $p_q^i = \text{softmax}(\hat{x}s^i)$, and select the optimal scaling factor $s{\hat{i}}$ minimizing the error:

\begin{equation}
\begin{aligned}
\hat{i} = \operatorname*{argmin}_{i} | p_f - p_q^i |, \quad i = 1,2,...,N.
\end{aligned}
\vspace{-1mm}
\end{equation}

\section{Experiment}


\begin{table*}[t]
  \small
  \centering
  \vspace{-0.3cm}
  \resizebox{0.95\textwidth}{!}{
\begin{tabular}{@{}lcc|l|l|ccccccc@{}}
  \toprule[1.5pt]
  \textbf{Backbone} & \textbf{resolution} & \textbf{Feat} &
  \textbf{Model} & \textbf{Method} &
  \textbf{mAP$\uparrow$} & \textbf{NDS$\uparrow$} & \textbf{mATE$\downarrow$} &
  \textbf{mASE$\downarrow$} & \textbf{mAOE$\downarrow$} &
  \textbf{mAVE$\downarrow$} & \textbf{mAAE$\downarrow$} \\
  \midrule
  \multirow{6}{*}{R50}
    & \multirow{6}{*}{\shortstack{512\\$\times$\\1408}}
    & \multirow{6}{*}{c5}
    & \multirow{3}{*}{PETR} & FP32   & 31.42 & 36.11 & 84.19  & 28.42 & 60.69 & 99.08 & 23.58 \\[1pt]
    & & &            & SQ    & 20.67 & 29.32 & 107.94 & 31.47 & 75.22 & 121.71 & 25.45 \\
    & & &            & Quarot& 22.81 & 30.00 & 104.87 & 30.99 & 74.16 & 118.54 & 25.07 \\ \cmidrule(l){4-12}
    & & & \multirow{3}{*}{\textbf{QF-PETR}} & FP32   & \textbf{31.49} & \textbf{37.20} & \textbf{82.52} & \textbf{27.88} & \textbf{59.91} & \textbf{91.74} & \textbf{23.45} \\[1pt]
    & & &                     & SQ*   & \textbf{31.34} & \textbf{37.17} & \textbf{82.61} & \textbf{27.93} & \textbf{60.00} & \textbf{91.79} & \textbf{23.45} \\
    & & &                     & Quarot* & \textbf{31.46} & \textbf{37.19} & \textbf{82.52} & \textbf{27.89} & \textbf{59.90} & \textbf{91.77} & \textbf{23.45} \\
  \midrule
  \multirow{6}{*}{V2-99}
    & \multirow{6}{*}{\shortstack{640\\$\times$\\1600}}
    & \multirow{6}{*}{p4}
    &  \multirow{3}{*}{StreamPETR} & FP32   & 49.51 & 58.03 & 60.10 & 26.07 & 35.65 & 25.91 & 19.60 \\[1pt]
    & & &             & SQ    & 18.72 & 35.66 & 74.32 & 30.39 & 41.49 & 30.53 & 20.82 \\
    & & &             & Quarot& 19.97 & 37.49 & 71.28 & 30.11 & 40.16 & 30.23 & 20.76 \\ \cmidrule(l){4-12}
    & & &  \multirow{3}{*}{\textbf{QF-StreamPETR}} & FP32   & \textbf{50.48} & \textbf{58.61} & \textbf{58.78} & \textbf{26.16} & \textbf{37.05} & \textbf{25.69} & \textbf{18.59} \\[1pt]
    & & &                           & SQ*   & \textbf{49.44} & \textbf{57.94} & \textbf{60.30} & \textbf{26.55} & \textbf{37.07} & \textbf{25.87} & \textbf{18.59} \\
    & & &                           & Quarot* & \textbf{50.12} & \textbf{58.39} & \textbf{58.86} & \textbf{26.16} & \textbf{37.05} & \textbf{25.87} & \textbf{18.59} \\
  \midrule
  \multirow{6}{*}{R50}
    & \multirow{6}{*}{\shortstack{512\\$\times$\\1408}}
    & \multirow{6}{*}{p4}
    &  \multirow{3}{*}{MV2d-T} & FP32 & 45.30 & 54.30 & 61.70 & 26.50 & 38.80 & 38.50 & 17.90 \\[1pt]
    & & &          & SQ    & 26.32 & 36.14 & 80.10 & 33.74 & 80.13 & 51.71 & 25.06 \\
    & & &          & Quarot& 28.17 & 39.25 & 78.46 & 32.11 & 79.21 & 49.19 & 23.57 \\ \cmidrule(l){4-12}
    & & &  \multirow{3}{*}{\textbf{QF-MV2d-T}} & FP32   & \textbf{46.38} & \textbf{54.91} & \textbf{60.36} & \textbf{26.24} & \textbf{37.61} & \textbf{37.98} & \textbf{17.68} \\[1pt]
    & & &                     & SQ*   & \textbf{45.13} & \textbf{52.71} & \textbf{60.44} & \textbf{26.36} & \textbf{37.62} & \textbf{38.12} & \textbf{17.93} \\
    & & &                     & Quarot* & \textbf{46.25} & \textbf{54.16} & \textbf{60.35} & \textbf{26.24} & \textbf{37.61} & \textbf{37.98} & \textbf{17.68} \\
  \midrule
  \multirow{6}{*}{V2-99}
    & \multirow{6}{*}{\shortstack{320\\$\times$\\800}}
    & \multirow{6}{*}{p4}
    &  \multirow{3}{*}{PETRv2} & FP32   & 41.00 & 50.30 & 72.26 & 26.92 & 45.29 & 38.93 & 19.33 \\[1pt]
    & & &             & SQ    & 15.12 & 31.04 & 96.14 & 30.11 & 45.71 & 52.14 & 27.11 \\
    & & &             & Quarot& 19.20 & 34.15 & 95.23 & 29.76 & 45.65 & 50.46 & 26.26 \\ \cmidrule(l){4-12}
    & & &  \multirow{3}{*}{\textbf{QF-PETRv2}} & FP32   & \textbf{41.86} & \textbf{51.03} & \textbf{71.03} & \textbf{26.15} & \textbf{44.86} & \textbf{37.54} & \textbf{19.04} \\[1pt]
    & & &                     & SQ*   & \textbf{40.73} & \textbf{50.61} & \textbf{71.76} & \textbf{27.02} & \textbf{45.24} & \textbf{37.97} & \textbf{19.42} \\
    & & &                     & Quarot* & \textbf{41.69} & \textbf{50.94} & \textbf{71.03} & \textbf{26.15} & \textbf{44.86} & \textbf{37.54} & \textbf{19.04} \\
  \bottomrule[1.5pt]
\end{tabular}}
  \vspace{-1mm}
  \caption{
    Comparison of floating-point and quantized performance on standard PETR-series models~\cite{streampetr,liu2022petr,liu2022petrv2}. 
    The default quantization setting is full-integer INT8 quantization.
    The DULUT configuration uses 32+32 entries.
    'SQ' denotes SmoothQuant~\cite{xiao2023smoothquant}, and 
    'Quarot' denotes rotation ptq methods ~\cite{ashkboos2024quarot}.
    An asterisk (*) indicates that Quantization-Aware Nonlinear Scaling (QANS) is additionally applied.
  }
  \label{tab:validation_on_petr_series_methods}
\vspace{-0.4cm}
\end{table*}

Details on the benchmark, metrics, and experimental settings are in Appendix C.

\subsection{Validation on PETR and its variants}\label{sec:sota_exp}
We evaluate our method comprehensively on PETR and its variants, covering both single-frame (PETR) and temporal multi-frame (PETRV2, MV2D, StreamPETR) settings, under floating-point (FP) and quantized conditions (see Tab.~\ref{tab:validation_on_petr_series_methods}). \textbf{Firstly}, we examine improvements in floating-point performance. Our method consistently improves mAP (ranging from 0.07 to 1.08 points) and significantly improves NDS (from 0.61 to 1.09 points) for both single-frame PETR models and multi-frame PETRV2, MV2D, and StreamPETR. \textbf{Secondly}, we analyze the quantization performance. For single-frame PETR models and temporal models (StreamPETR, PETRV2, MV2d), our method effectively constrains accuracy degradation to less than 1\% in both mAP and NDS metrics, thanks to the proposed QFPE, DULUT and QANS. \textbf{Finally}, additional experiments on PETR-series models with diverse backbones and input resolutions are provided in the Appendix D.

\subsection{Ablation Study}\label{sec:exp_ablation}
\textbf{Ablation for different quantization methods.} We evaluate the Camera-ray PE module on the nuScenes dataset under three configurations: FP32 Baseline (full precision as an upper bound),  8-bit PTQ Methods SmoothQuant \cite{xiao2023smoothquant} and Quarot \cite{ashkboos2024quarot}. As shown in Table~\ref{tab:quant_softmax_ptq4vit_comparison}, retaining the Softmax input in full precision (“No”) yields higher mAP and NDS than when it is quantized (“Yes”), underscoring the importance of careful Softmax treatment.

\setlength{\tabcolsep}{1.8mm}
\begin{table}[htb]
    \vspace{-2mm}
    \small
    \begin{tabular}{l|c|cc|cc}
    \toprule[1.5pt]
    \multirow{2}{*}{Method} & 
    \multirow{2}{*}{\begin{tabular}[c]{@{}c@{}}Quant.Softmax\\ Input\end{tabular}} & 
    \multicolumn{2}{c|}{(SQ) INT8} & 
    \multicolumn{2}{c}{(Quarot) INT8} \\
    \cline{3-6}
    & & mAP$\uparrow$ & NDS$\uparrow$ & mAP$\uparrow$ & NDS$\uparrow$ \\
    \midrule
    \multirow{3}{*}{{\begin{tabular}[c]{@{}c@{}}Camera-\\ray PE\end{tabular}}}
    & FP32 & 31.42 & 36.11 & 31.42 & 36.11 \\
    & No   & 24.90 & 32.10 & 27.10 & 33.60 \\
    & Yes  & 20.67 & 29.32 & 22.81 & 30.00 \\
    \bottomrule[1.5pt]
    \end{tabular}
    \vspace{-0.3cm}
    \caption{Quantization performance of Camera-ray PE on nuScenes. FP32, 8-bit SmoothQuant and Quarot methods are compared under different Softmax quantization settings.}
    \label{tab:quant_softmax_ptq4vit_comparison}
    \vspace{-2mm}
\end{table}

\noindent\textbf{Effect of Anchor Embedding Quantity.}
The QFPE uses three anchor embeddings per axis, obtained through linear interpolation. Experiments (Tab.~\ref{tab:sd}) demonstrate that setting the number of anchor embeddings to 3 achieves the highest NDS and mAP scores. Adjusting this number either up or down results in lower performance, confirming that 3 is the optimal choice.

\setlength{\tabcolsep}{3.4mm}
\begin{table}[htb] 
\vspace{-2mm}
    \centering
    \begin{tabular}{c|c|c|cc}
    \toprule[1.5pt]
    \multicolumn{3}{c|}{Quantity of Anchor Embedding} & \multirow{2}{*}{NDS$\uparrow$} & \multirow{2}{*}{mAP$\uparrow$}     \\ \cline{1-3}
    \multicolumn{1}{c|}{x-axis} & \multicolumn{1}{c|}{y-axis} & \multicolumn{1}{c|}{z-axis}  &  \\ 
    \midrule
    2 & 2 & 2 & 36.66 & 31.29 \\
    3 & 3 & 3 & \textbf{37.20} & \textbf{31.49} \\
    4 & 4 & 4 & 36.92 & 31.09 \\
    5 & 5 & 5 & 36.83 & 31.19 \\
    \bottomrule[1.5pt]
    \end{tabular}
    \vspace{-1mm}
    \caption{
    Effect of Anchor Embedding Quantity.
    }
    \label{tab:sd}
    \vspace{-5mm}
\end{table}

\noindent\textbf{Proof of Position embedding Equivalence.}
We conducted experiments to verify whether the proposed QFPE enhances floating-point performance over the original camera-ray PE. As shown in Tab.~\ref{tab:floating_point_performance_comparison_of_different_3D_position_embedding}, QFPE provides performance improvements. On PETR, it slightly increases mAP by 0.07 but significantly boosts NDS and mATE by 1.09 and 1.67, respectively. For Stream-PETR, our method yields substantial and balanced enhancements, with increases of 0.94 in mAP, 0.46 in NDS, and 0.22 in mATE.

\begin{figure}[ht]
\vspace{-1mm}
\centering	
\includegraphics[width=1\linewidth]{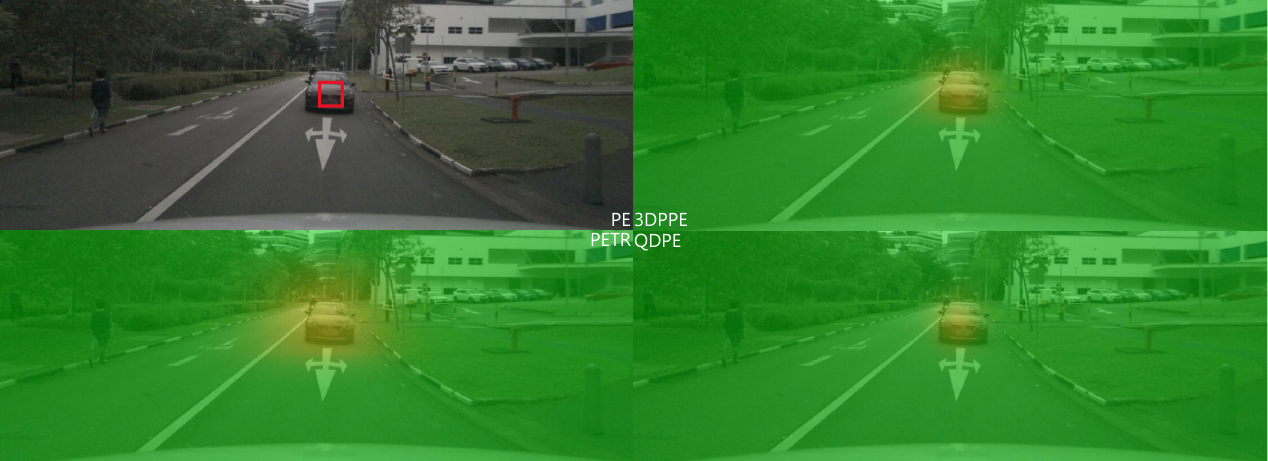}
\vspace{-1mm}
\caption{Qualitative comparison of the local similarity.}
\label{fig:pe_visual_compare}
\vspace{-2mm}
\end{figure}

\setlength{\tabcolsep}{2.9mm}
\begin{table}[htb] 
    \small
    \begin{tabular}{lc|ccc}
    \toprule[1.5pt]
    \multicolumn{2}{c|}{Method} & mAP$\uparrow$ & NDS$\uparrow$ & mATE$\downarrow$ \\
    \noalign{\smallskip}
    \hline
    \noalign{\smallskip}
    \multirow{2}{*}{{\begin{tabular}[c]{@{}c@{}}PETR\end{tabular}}}
    & Camera-ray PE                  & 31.42 & 36.11 & 84.19  \\ 
    & QFPE          & \textbf{31.49} & \textbf{37.20} & \textbf{82.52}  \\
    \hline
    \multirow{2}{*}{{\begin{tabular}[c]{@{}c@{}}Stream\\-PETR\end{tabular}}}
    & Camera-ray PE                  & 
    49.51 & 58.03 & 60.10  \\ 
    & QFPE          & \textbf{50.48} & \textbf{58.61} & \textbf{58.78}  \\
    \bottomrule[1.5pt]
    \end{tabular}
    \vspace{-1mm}
    \caption{FP Performance of different 3D PE.}
\label{tab:floating_point_performance_comparison_of_different_3D_position_embedding}
    \vspace{-4mm}
\end{table}

\textbf{Local Similarity of PE Features.} Figure ~\ref{fig:pe_visual_compare} shows that QFPE significantly outperforms 3D point PE and cameraray PE in local similarity of position embedding. Its similarity distribution appears more compact and concentrated, validating the method's superiority in local spatial information modeling and its capability to precisely capture neighborhood spatial relationships around target pixels.

\textbf{Compare with QAT.} Although QFPE requires retraining, its cost remains comparable to QAT.
We implement a distillation-based QAT (inspired by QD-BEV~\cite{zhang2023qd}) on the original Camera-ray PE.
As shown in Table~\ref{tab:qat_compare}, QFPE consistently outperforms QAT in both mAP and NDS metrics, despite QAT's longer training epochs (24 or 36 epochs).
This demonstrates the advantage of our amplitude-aware design in preserving floating-point performance and improving quantized accuracy.

\setlength{\tabcolsep}{3.7mm}  
\begin{table}[ht]
    \vspace{-2mm}
    \small
    \centering
    \begin{tabular}{c|cc|cc}
    \toprule[1.5pt]
    Epochs & \multicolumn{2}{c|}{QAT (Distill)} & \multicolumn{2}{c}{QFPE (PTQ)} \\
    \cline{2-5}
     & mAP$\uparrow$ & NDS$\uparrow$ & mAP$\uparrow$ & NDS$\uparrow$ \\
    \noalign{\smallskip}
    \hline
    \noalign{\smallskip}
    12 & 28.9 & 35.2 & \textbf{31.46} & \textbf{37.19} \\
    24 & 30.3 & 35.8 & \textbf{31.46} & \textbf{37.19} \\
    36 & 30.3 & 35.8 & \textbf{31.46} & \textbf{37.19} \\
    \bottomrule[1.5pt]
    \end{tabular}
    \vspace{-0.2cm}
    \caption{Comparison of QAT with distillation vs.\ QFPE PTQ across training epochs on the nuScenes dataset.}
    \label{tab:qat_compare}
    \vspace{-2mm}
\end{table}

\noindent\textbf{Impact of Different Scaled Dot Product Quantization Strategy.}
To isolate the effect of our scaled dot-product quantization, we quantize only the softmax inputs and keep all other modules in FP. As shown in Tab.~\ref{tab:performance_of_dfferent_scaled_dot_product_quantization_strategies}, the original strategy causes large drops (about 40\% NDS and 50\% mAP). Our “quant after stabilization” improves accuracy. Ablating the truncation range $N$ shows that $N \geq 20$ is near-lossless, $N < 20$ harms accuracy due to over-truncation, and larger $N$ brings no further gain. We therefore set $N=20$.

\setlength{\tabcolsep}{3.4mm}
\begin{table}[thb]
    \vspace{-3mm}
    \small
    \centering
    \begin{tabular}{lc|p{0.6cm}p{0.6cm}p{0.8cm}}
    \toprule[1.5pt]
    \multicolumn{2}{c|}{Method} & NDS$\uparrow$ & mAP$\uparrow$ & mATE$\downarrow$ \\
    \noalign{\smallskip}
    \hline
    \noalign{\smallskip}
    \multirow{1}{*}{{\begin{tabular}[c]{@{}c@{}}quant before ns\end{tabular}}}
    & -          & 25.31 & 13.79 & 107.94  \\
    \hline
    \multirow{6}{*}{{\begin{tabular}[c]{@{}c@{}}quant after ns\end{tabular}}}
    & N = 1\textcolor{white}{0}  & 3.45 & 1.23 & 150.34  \\ 
    & N = 5\textcolor{white}{0}  & 33.86 & 28.77 & 87.12  \\
    & N = 10 & 34.65 & 29.33 & 85.01  \\
    & N = 20 & 36.10 & 31.42 & 84.19  \\
    & N = 30 & 36.10 & 31.42 & 84.19  \\
    & N = 40 & 36.10 & 31.42 & 84.19  \\
    \bottomrule[1.5pt]
    \end{tabular}
    \vspace{-1mm}
    \caption{Performance of Different Scaled Dot Product Quantization Strategies.}
    \label{tab:performance_of_dfferent_scaled_dot_product_quantization_strategies}
    \vspace{-4mm}
\end{table}

\noindent\textbf{Superior Performance of DULUT for Non-linear Functions.}
Using "quant after stabilization (N=20)" as the baseline (Tab.~\ref{tab:performance_of_dfferent_scaled_dot_product_quantization_strategies}), we compare I-BERT, I-ViT, standard linear LUTs, and our DULUT (Tab.~\ref{tab:quantization_methods_for_nonlinear_functions}). A linear LUT requires 256 entries for lossless accuracy; with 128 entries it drops by 0.54 NDS and 0.37 mAP. DULUT achieves lossless accuracy with 128 entries and remains near-lossless with 64 entries (-0.08\% NDS, -0.02\% mAP), confirming its efficiency for SiLU, GELU, and Softmax.

\vspace{-1mm}
\setlength{\tabcolsep}{2mm}
\begin{table}[htb] 
    \centering
    \small
    \begin{tabular}{lc|ccc}
    \toprule[1.5pt]
    \multicolumn{2}{c|}{Method} & NDS$\uparrow$ & mAP$\uparrow$ & mATE$\downarrow$ \\
    \hline
    \multicolumn{2}{c|}{Quant after stabilization (N=20)} & 36.10  & 31.42 & 84.19 \\
    \hline
    \multicolumn{2}{c|}{I-Bert} & 34.87 & 29.34 & 88.41 \\
    \hline
    \multicolumn{2}{c|}{I-Vit} & 35.03 & 28.77 & 87.32 \\
    \hline
    \multirow{2}{*}{{\begin{tabular}[c]{@{}c@{}}LUT\end{tabular}}}
    & 256 entries & 36.10 & 31.42 & 84.19  \\ 
    & 128 entries & 35.56 & 31.05 & 85.61  \\
    \hline
    \multirow{4}{*}{{\begin{tabular}[c]{@{}c@{}}DULUT\end{tabular}}}
    & (16,16) entries & 28.12 & 17.36 & 96.99  \\ 
    & (16,32) entries & 34.14 & 27.29 &  90.33 \\
    & (32,32) entries & 36.07 & 31.36 & 84.20  \\
    & (64,64) entries & 36.10 & 31.42 & 84.19  \\
    \bottomrule[1.5pt]
    \end{tabular}
    \vspace{-2mm}
    \caption{Performance comparison of different quantization methods for nonlinear functions.}
\label{tab:quantization_methods_for_nonlinear_functions}
\vspace{-4mm}
\end{table}

\textbf{Practical Hardware Resource Savings.}
Tab.~\ref{tab:table} shows Q-PETR runs at 27.6 FPS (3.9$\times$ faster) and 1.3\,GB memory (75\% less) vs.\ PETR’s 7.1\,FPS/4.8\,GB, demonstrating significant speedup and resource efficiency.

\setlength{\tabcolsep}{3.6mm}
\begin{table}[htb]
\centering
\resizebox{0.95\linewidth}{!}{
\begin{tabular}{lccc}
\toprule[1.5pt]
Method & Mode & FPS & CUDA Memory (G) \\
\hline
PETR & FP32 & 7.1 & 4.8 \\
FQ-PETR & INT8 & 27.6 & 1.3 \\
\bottomrule[1.5pt]
\end{tabular}}
  \vspace{-0.2cm}
\caption{\small FPS and CUDA Memory Comparison: PETR vs. Q-PETR (R50-DCN, 512×1408, RTX 4090).}
\label{tab:table}
\vspace{-0.5cm}
\end{table}

\vspace{-1mm}
\section{Conclusion}
We introduce FQ-PETR, a fully quantized framework specifically designed for the efficient deployment of PETRs on edge hardware. 
By addressing the magnitude imbalance between positional embeddings and image features, our proposed QFPE significantly reduces quantization difficulty without sacrificing accuracy. 
We further develop DULUT, a hardware-friendly algorithm for accurately approximating nonlinear functions with minimal resources. 
Additionally, we propose QANS, effectively mitigating the attention distortion issue from quantizing extreme softmax inputs. 
Extensive experiments confirm that FQ-PETR achieves near-floating-point accuracy with substantially improved inference efficiency, demonstrating its practical value for real-world autonomous driving or computer vision 
applications.

\vspace{-1mm}
\section{Limitations and Future Work}
This work targets PETR-style 3D detectors under standard benchmarks, INT8 settings, and common LUT-based accelerators. We have not studied large-model pipelines that use PETR as a 3D feature encoder, nor broad cross-dataset/camera-rig variations. Future work will extend FQ-PETR to these scenarios, explore lower/mixed precision, reduce the finetuning cost of QFPE, and investigate the integration with privacy-preserving federated learning~\cite{liu2024fedbcgd,liu2025consistency,liu2025improving}.

\bibliography{main}

\clearpage
\include{appendix}

\end{document}

%% file: appendix.tex
\appendix
\setcounter{page}{1}
{\section*{Appendix}\huge }

\section*{A \quad Theoretical Analysis of Magnitude Bounds in Position embeddings}
\label{sec:mag_analysis}

\subsection{Normalization Framework and Input Conditioning}
\label{subsec:normalization}
To establish a unified analytical framework, we first formalize the spatial normalization process for various ray-based position embeddings. Let $\mathbf{p} = (x, y, z)$ denote the 3D coordinates within the perception range $x, y \in [-51.2, 51.2]$ meters and $z \in [-5, 3]$ meters. The normalized coordinates $\mathbf{v} \in [0, 1]^3$ are computed as:

\begin{equation}
    \mathbf{v} = \left( \frac{x + 51.2}{102.4}, \frac{y + 51.2}{102.4}, \frac{z + 5.0}{8.0} \right)
    \label{eq:normalization}
\end{equation}

Noting that $\mathbf{v}$ is clamped to $\mathbf{v}_c$ within the range $[0, 1]$, the distribution ranges of the normalized sampled points in positional embeddings are characterized as follows:
\begin{itemize}
    \item For the sampled point of Camera-Ray PE, denoted as $\mathbf{v}_c^{CR}$, the distribution spans the unit cube, i.e., $[0, 1] \times [0, 1] \times [0, 1]$.
    \item For the sampled points of LiDAR-Ray PE and QFPE, denoted as $\mathbf{v}_c^{LR}$ and $\mathbf{v}_c^{QD}$ respectively, the distributions are constrained to $[0, 0.79] \times [0, 0.79] \times [0, 1]$.
\end{itemize}
Here, the value $0.79$ is derived from the ratio $30/51.2$, where $30$ corresponds to the fixed depth setting in the embedding process. This distinction highlights the inherent differences in spatial coverage and normalization strategies employed by these positional embeddings.

\subsection{Magnitude Propagation Analysis}
\label{subsec:magnitude_propagation}

\subsubsection{Camera-Ray Position embedding}
\label{subsubsec:camera_pe}
As illustrated in Fig.~\ref{fig:pe_compare} (a), the embedding pipeline consists of two critical stages:

\textbf{Stage 1: Inverse Sigmoid Transformation}
\begin{equation}
    \hat{\mathbf{v}}^{CR} = \ln\left(\frac{\mathbf{v}_c^{CR} + \epsilon}{1 - (\mathbf{v}_c^{CR} + \epsilon)}\right), \quad \epsilon = 10^{-5}
    \label{eq:logit_transform_cr}
\end{equation}
Empirical analysis reveals a maximum magnitude $\eta_{\text{max}} = \max(\|\hat{\mathbf{v}}^{CR}\|_\infty) \approx 11.5$.

\textbf{Stage 2: MLP Projection} (Through Two Fully-Connected Layers)
\begin{equation}
    \text{PE}_{\text{CR}} = \mathbf{W}_2 \sigma(\mathbf{W}_1 \hat{\mathbf{v}}^{CR} + \mathbf{b}_1) + \mathbf{b}_2
    \label{eq:mlp_transform_cr}
\end{equation}
where $\sigma$ denotes the ReLU activation function. Let $\Gamma = \max(\|\mathbf{W}_1\|_{\max}, \|\mathbf{W}_2\|_{\max})$ be the maximum weight magnitude. We derive the upper bound:
\begin{equation}
    \| \text{PE}_{\text{CR}} \|_\infty \leq 256 \cdot 192 \cdot \Gamma^2 \cdot 11.5
    \label{eq:cr_bound}
\end{equation}
where $192$ and $256$ denote the input tensor channels for $\mathbf{W}_1$ and $\mathbf{W}_2$, respectively.

\subsubsection{Ours QFPE}
\label{subsubsec:qd_pe}
The proposed embedding introduces anchor-based constraints, as depicted in Fig.~\ref{fig:pe_compare} (c):

\textbf{Stage 1: Anchor Interpolation} (For Each Axis $\alpha \in \{x, y, z\}$)
\begin{equation}
    \mathbf{e}_\alpha = \frac{p_\alpha - L_\alpha^i}{\Delta L_\alpha} \mathbf{E}_\alpha^{i+1} + \frac{L_\alpha^{i+1} - p_\alpha}{\Delta L_\alpha} \mathbf{E}_\alpha^i
    \label{eq:anchor_interp}
\end{equation}
where $\mathbf{E}_\alpha^i$ denotes learnable anchor embeddings.
Via Theorem~\ref{thm:anchor}, the magnitude is constrain to:
\begin{equation}
    \| \mathbf{e}_\alpha \|_\infty \leq \gamma \quad 
    \label{eq:anchor_bound}
\end{equation}

\textbf{Stage 2: MLP Projection}
\begin{equation}
    \| \text{PE}_{\text{QD}} \|_\infty \leq 256 \cdot 192 \cdot \Gamma^2 \cdot 2.6
    \label{eq:qd_bound}
\end{equation}

\subsection{Comparative Magnitude Analysis}
\label{subsec:comparative}
The derived bounds reveal fundamental differences in magnitude scaling:
\begin{align}
    \frac{\| \text{PE}_{\text{CR}} \|}{\| \text{PE}_{\text{QD}} \|} &\approx \frac{11.5}{2.6} = 4.4
    \label{eq:ratio_qd}
\end{align}
This analysis demonstrates that QFPE requires $4\times$ less quantization range than Camera-Ray PE.

\subsection{Theoretical Guarantee of Magnitude Constraints}
\label{subsec:theorem}
\begin{theorem}[Anchor Embedding Magnitude Bound]
\label{thm:anchor}
Let $\mathbf{E}_\alpha^i, \mathbf{E}_\alpha^{i+1}$ be adjacent anchor embeddings with $\|\mathbf{E}_\alpha^i\|_\infty \leq \gamma$. For any point $p_\alpha \in [L_\alpha^i, L_\alpha^{i+1}]$, its interpolated embedding satisfies:
\begin{equation}
    \| \mathbf{e}_\alpha \|_\infty \leq \gamma
\end{equation}
\end{theorem}

\begin{proof}
Let $\lambda = \frac{p_\alpha - L_\alpha^i}{\Delta L_\alpha} \in [0, 1]$. The interpolated embedding becomes:
\begin{equation}
    \mathbf{e}_\alpha = \lambda \mathbf{E}_\alpha^{i+1} + (1 - \lambda) \mathbf{E}_\alpha^i
\end{equation}
For any component $k$:
\begin{equation}
    |e_{\alpha,k}| \leq \lambda |E_{\alpha,k}^{i+1}| + (1 - \lambda) |E_{\alpha,k}^i| \leq \lambda \gamma + (1 - \lambda) \gamma = \gamma
\end{equation}
Thus, $\|\mathbf{e}_\alpha\|_\infty \leq \gamma$ holds for all dimensions.
\end{proof}
Through the application of regularization (e.g., L2 constraint) on the anchor embeddings $\mathbf{E}_\alpha^i$
during training, the magnitude of $\gamma$ can be explicitly controlled.
Empirically, we find that this value converges to approximately 0.8 in our experiments.

\section*{B \quad DULUT Under Hardware Constraints}
\label{sec:appendix_dulut}

\subsection*{B.1 \; Why curvature-only merging fails}
For a twice–differentiable $f$, linear interpolation on $[x_i,x_{i+1}]$ obeys
\begin{equation}
\label{eq:lin-int-bound}
\begin{aligned}
\max_{x\in[x_i,x_{i+1}]} \bigl|f(x)-P(x)\bigr|
&\le
\frac{(x_{i+1}-x_i)^2}{8}\,
\max_{x\in[x_i,x_{i+1}]} \bigl|f''(x)\bigr|
\\
&\quad + \varepsilon_{\mathrm{hw}},
\end{aligned}
\end{equation}
where $\varepsilon_{\mathrm{hw}}$ captures rounding and finite precision. If we
ignore $\varepsilon_{\mathrm{hw}}$ and only use curvature, we shrink intervals
where $\lvert f''\rvert$ is large and merge where it is near zero. This can fail
on real hardware due to fixed-point arithmetic and fixed per-segment resolution.

\subsection*{B.2 \; Hardware interpolation model}
Consider signed INT8 inputs $i\in[-128,127]$ and a linear LUT with $T$ entries.
Let $i\_bit=8$ and $t\_bit=\log_2 T$. A common fixed-point blend is
\begin{align}
\text{shift\_bit} &= i\_bit - t\_bit, \\
\text{shift\_num} &= 2^{\text{shift\_bit}}, \\
\text{idx} &= \left\lfloor \frac{i+128}{\text{shift\_num}} \right\rfloor, \\
p &= (i+128)\bmod \text{shift\_num}, \\
\hat{y}(i) &=
\frac{(\text{shift\_num}-p)\, \text{table}[\text{idx}]
+ p\, \text{table}[\text{idx}+1]}{\text{shift\_num}}.
\end{align}
\textbf{Key point.} Within any LUT segment there are only $\text{shift\_num}$
distinct interpolation weights before rounding; thus a segment yields at most
$\text{shift\_num}$ distinct blended values.

\paragraph{Collision lemma (informal).}
If a segment covers more than $\text{shift\_num}$ integer codes, at least two
different inputs will map to the same output after blending and rounding.

\subsection*{B.3 \; SiLU example (INT8, $T{=}32$, resolution $=8$)}
Let $\mathrm{SiLU}(x)=x\,\sigma(x)$. Suppose $i\_bit{=}8$, $T{=}32$ (so
$t\_bit{=}5$), hence $\text{shift\_bit}{=}3$ and $\text{shift\_num}{=}8$.
With scale $s=1/2$, one integer step equals $0.5$ in real value.
Consider $i\in[16,24]$ where $\mathrm{SiLU}''(x)\approx 0$. A curvature-only
rule would merge $[16,20]$ and $[20,24]$ into $[16,24]$ to save entries. But
$[16,24]$ contains $9$ integer codes while the hardware offers only $8$
interpolation positions, so at least one adjacent pair collides (same output).
Since $\mathrm{SiLU}$ is not exactly linear, these collisions create visible
“steps” even when $f''\!\approx 0$.

\subsection*{B.4 \; Practical takeaway}
Curvature-based merging must respect hardware resolution: a segment should not
span more than $\text{shift\_num}$ integer codes. Our \emph{DULUT} uses
error-driven split/merge under these constraints: initialize two cascaded
linear LUTs (equivalent to one), measure per-segment ARE, merge the smallest-ARE
segment and split the largest-ARE segment only if global error decreases and
the span stays within the effective resolution. With INT8, two $32$-entry tables
match a much larger single-table LUT and remain compiler-friendly (standard
ONNX chains fused by common backends).

\section*{C \quad Experimental Setup}\label{sec:exp_setup}
\textbf{Benchmark.}
We use the nuScenes dataset, a comprehensive autonomous driving dataset covering object detection, tracking, and LiDAR segmentation. The vehicle is equipped with one LiDAR, five radars, and six cameras providing a 360-degree view. The dataset comprises 1,000 driving scenes split into training (700 scenes), validation (150 scenes), and testing (150 scenes) subsets. Each scene lasts 20 seconds, annotated at 2 Hz.

\textbf{Metrics.}
Following the official evaluation protocol, we report the nuScenes Score (NDS), mean Average Precision (mAP), and five true positive metrics: mean Average Translation Error (mATE), Scale Error (mASE), Orientation Error (mAOE), Velocity Error (mAVE), and Attribute Error (mAAE).

\textbf{Experimental Details.}
Our experiments encompass both floating-point training and quantization configurations. For floating-point training, we follow PETR series settings, using PETR with an R50dcn backbone unless specified, and utilize the C5 feature (1/32 resolution output) as the 2D feature. Input images are at $1408 \times 512$ resolution. Both the lidar-ray PE and QD-aware lidar-ray PE use a pixel-wise depth of 30m with three anchor embeddings per axis. The 3D perception space is defined as $[-61.2, 61.2]$m along the X and Y axes, and $[-10, 10]$m along the Z axis. We also compare these positional embeddings on StreamPETR, using a V2-99 backbone and input images of $800 \times 320$ resolution.

Training uses the AdamW optimizer (weight decay 0.01) with an initial learning rate of $2.0 \times 10^{-4}$, decayed via a cosine annealing schedule. We train for 24 epochs with a batch size of 8 on four NVIDIA RTX 4090 GPUs. No test-time augmentation is applied.

For quantization, we adopt 8-bit symmetric per-tensor post-training quantization, using 32 randomly selected training images for calibration. When quantizing the scaled dot-product in cross-attention, we define a candidate set of 20 scaling factors.

\setlength{\tabcolsep}{3.5mm}
\begin{table}[thb]
    \small
    \begin{tabular}{lc|p{0.6cm}p{0.6cm}p{0.8cm}}
    \toprule[1.5pt]
    \multicolumn{2}{c|}{Method} & NDS$\uparrow$ & mAP$\uparrow$ & mATE$\downarrow$ \\
    \noalign{\smallskip}
    \hline
    \noalign{\smallskip}
    \multirow{2}{*}{{\begin{tabular}[c]{@{}c@{}}PETR\end{tabular}}}
    & Camera-ray PE & 34.29 & 27.66 & 87.17  \\ 
    & QFPE          & \textbf{37.18} & \textbf{31.40} & \textbf{82.59}  \\
    \hline
    \multirow{2}{*}{{\begin{tabular}[c]{@{}c@{}}Stream\\-PETR\end{tabular}}}
    & Camera-ray PE & 53.74 & 40.23 & 69.39  \\ 
    & QFPE          & \textbf{56.81} & \textbf{47.65} & \textbf{61.53}  \\
    \bottomrule[1.5pt]
    \end{tabular}
    \vspace{-0.3cm}
    \caption{Quantization Performance Comparison of different 3D position embedding.}
    \label{tab:quantization_performance_comparison_of_different_3D_position_embedding}
    \vspace{-0.5cm}
\end{table}

\begin{table}[htb] 
    \centering
    \resizebox{0.99\linewidth}{!}{
    \footnotesize
    \begin{tabular}{l|cc|cc}
    \toprule[1.5pt]
    \multicolumn{1}{c|}{Model Name} 
    & \multicolumn{2}{c|}{qwen2.5-7b-instruct} 
    & \multicolumn{2}{c}{deepseek-r1-distill-qwen-7b} \\
    \noalign{\smallskip}
    \hline
    \noalign{\smallskip}
    & wikitext2$\downarrow$ & gsm8k$\uparrow$ & wikitext2$\downarrow$ & gsm8k$\uparrow$ \\
    \noalign{\smallskip}
    \hline
    \noalign{\smallskip}
    bfp16 & 7.46 & 80.21 & 25.04 & 85.97 \\
    quant before ns & 10000+ & 0.3 & 10000+ & 0.1 \\
    quant after ns(20) & 7.48 & 80.24 & 25.09 & 86.03 \\
    \bottomrule[1.5pt]
    \end{tabular}}
    \vspace{-2mm}
    \caption{Quant after ns in LLMs}
    \label{tab:comparison_qwen_deepseek}
    \vspace{-3mm}
\end{table}

\begin{figure}[htb]
\centering
	\includegraphics[width=1.0\linewidth]{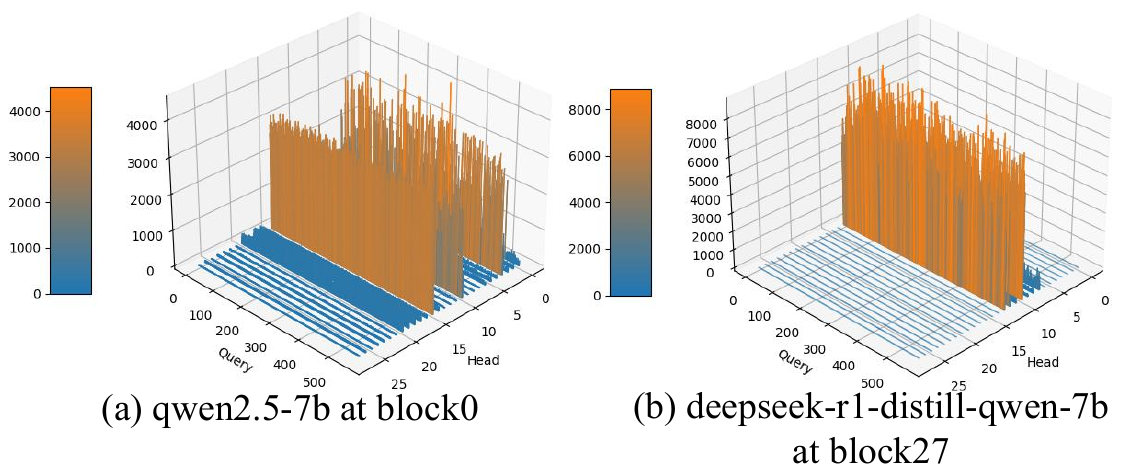}
    \vspace{-6mm}
	\caption{Softmax input distributions from two large language models (qwen2.5-7b at block0 on the left, and deepseek-r1-distill-qwen-7b at block27 on the right).}
	\label{fig:llm_softmax}
    \vspace{-0.5cm}
\end{figure}

\begin{table*}[t]
\setlength{\tabcolsep}{3.2pt}
\centering
\small
\begin{tabular}{l c c l l c c c c c c c}
\toprule
\textbf{Bac} & \textbf{size} & \textbf{Feat} & \textbf{Model} & \textbf{Method} &
\textbf{mAP}$\uparrow$ & \textbf{NDS}$\uparrow$ & \textbf{mATE}$\downarrow$ &
\textbf{mASE}$\downarrow$ & \textbf{mAOE}$\downarrow$ & \textbf{mAVE}$\downarrow$ &
\textbf{mAAE}$\downarrow$ \\
\midrule
\multirow{6}{*}{R50}
  & \multirow{6}{*}{\shortstack{512\\$\times$\\1408}}
  & \multirow{6}{*}{c5}
  & \multirow{3}{*}{PETR}
    & FP32     & 31.42 & 36.11 & 84.19 & 28.42 & 60.69 & 99.08  & 23.58 \\
  &  &  & & SQ       & 20.67 & 29.32 & 107.94 & 31.47 & 75.22 & 121.71 & 25.45 \\
  &  &  & & Quarot   & 22.81 & 30.00 & 104.87 & 30.99 & 74.16 & 118.54 & 25.07 \\
  &  &  & \multirow{3}{*}{\textbf{FQ-PETR}}
    & FP32     & \textbf{31.49} & \textbf{37.20} & \textbf{82.52} & \textbf{27.88} & \textbf{59.91} & \textbf{91.74}  & \textbf{23.45} \\
  &  &  & & SQ\textsuperscript{*}     & \textbf{31.34} & \textbf{37.17} & \textbf{82.61} & \textbf{27.93} & \textbf{60.00} & \textbf{91.79}  & \textbf{23.45} \\
  &  &  & & Quarot\textsuperscript{*} & \textbf{31.46} & \textbf{37.19} & \textbf{82.52} & \textbf{27.89} & \textbf{59.90} & \textbf{91.77}  & \textbf{23.45} \\
\midrule
\multirow{6}{*}{R50}
  & \multirow{6}{*}{\shortstack{512\\$\times$\\1408}}
  & \multirow{6}{*}{p4}
  & \multirow{3}{*}{PETR}
    & FP32     & 32.60 & 71.16 & 82.63 & 27.96 & 61.06 & 95.81  & 23.91 \\
  &  &  & & SQ       & 12.97 & 24.75 & 108.28 & 31.76 & 79.57 & 78.90  & 27.14 \\
  &  &  & & Quarot   & 15.12 & 25.34 & 103.66 & 30.89 & 78.11 & 77.23  & 26.60 \\
  &  &  & \multirow{3}{*}{\textbf{FQ-PETR}}
    & FP32     & \textbf{32.69} & \textbf{38.03} & \textbf{80.58} & \textbf{27.89} & \textbf{59.43} & \textbf{92.69}  & \textbf{22.55} \\
  &  &  & & SQ\textsuperscript{*}     & \textbf{32.40} & \textbf{37.72} & \textbf{81.11} & \textbf{27.92} & \textbf{60.02} & \textbf{92.76}  & \textbf{22.59} \\
  &  &  & & Quarot\textsuperscript{*} & \textbf{32.67} & \textbf{37.99} & \textbf{80.58} & \textbf{27.89} & \textbf{59.43} & \textbf{92.69}  & \textbf{22.55} \\
\midrule
\multirow{6}{*}{V2-99}
  & \multirow{6}{*}{\shortstack{640\\$\times$\\1600}}
  & \multirow{6}{*}{p4}
  & \multirow{3}{*}{PETR}
    & FP32     & 40.66 & 46.50 & 71.76 & 27.07 & 42.23 & 80.68  & 21.06 \\
  &  &  & & SQ       & 12.40 & 26.98 & 117.38 & 34.85 & 83.38 & 106.83 & 27.10 \\
  &  &  & & Quarot   & 13.11 & 27.13 & 112.87 & 32.43 & 81.11 & 104.29 & 26.76 \\
  &  &  & \multirow{3}{*}{\textbf{FQ-PETR}}
    & FP32     & \textbf{41.35} & \textbf{45.99} & \textbf{72.18} & \textbf{26.91} & \textbf{45.05} & \textbf{82.03}  & \textbf{20.67} \\
  &  &  & & SQ\textsuperscript{*}     & \textbf{40.95} & \textbf{45.64} & \textbf{73.40} & \textbf{27.05} & \textbf{45.61} & \textbf{102.17} & \textbf{20.68} \\
  &  &  & & Quarot\textsuperscript{*} & \textbf{41.31} & \textbf{45.86} & \textbf{72.20} & \textbf{26.91} & \textbf{45.05} & \textbf{101.03} & \textbf{20.67} \\
\midrule
\multirow{6}{*}{R50}
  & \multirow{6}{*}{\shortstack{512\\$\times$\\1408}}
  & \multirow{6}{*}{p4}
  & \multirow{3}{*}{MV2d-T}
    & FP32     & 45.30 & 54.30 & 61.70 & 26.50 & 38.80 & 38.50  & 17.90 \\
  &  &  & & SQ       & 26.32 & 36.14 & 80.10 & 33.74 & 80.13 & 51.71  & 25.06 \\
  &  &  & & Quarot   & 28.17 & 39.25 & 78.46 & 32.11 & 79.21 & 49.19  & 23.57 \\
  &  &  & \multirow{3}{*}{\textbf{FQ-MV2d-T}}
    & FP32     & \textbf{46.38} & \textbf{54.91} & \textbf{60.36} & \textbf{26.24} & \textbf{37.61} & \textbf{37.98}  & \textbf{17.68} \\
  &  &  & & SQ\textsuperscript{*}     & \textbf{45.13} & \textbf{52.71} & \textbf{60.44} & \textbf{26.36} & \textbf{37.62} & \textbf{38.12}  & \textbf{17.93} \\
  &  &  & & Quarot\textsuperscript{*} & \textbf{46.25} & \textbf{54.16} & \textbf{60.35} & \textbf{26.24} & \textbf{37.61} & \textbf{37.98}  & \textbf{17.68} \\
\midrule
\multirow{6}{*}{V2-99}
  & \multirow{6}{*}{\shortstack{320\\$\times$\\800}}
  & \multirow{6}{*}{p4}
  & \multirow{3}{*}{StreamPETR}
    & FP32     & 48.19 & 57.11 & 60.99 & 25.58 & 37.54 & 26.28  & 19.43 \\
  &  &  & & SQ       & 18.52 & 36.47 & 76.39 & 31.44 & 47.03 & 30.22  & 20.99 \\
  &  &  & & Quarot   & 21.33 & 38.13 & 75.65 & 30.74 & 46.21 & 29.47  & 20.21 \\
  &  &  & \multirow{3}{*}{\textbf{FQ-StreamPETR}}
    & FP32     & \textbf{49.13} & \textbf{57.57} & \textbf{60.77} & \textbf{26.14} & \textbf{39.05} & \textbf{24.81}  & \textbf{19.15} \\
  &  &  & & SQ\textsuperscript{*}     & \textbf{48.21} & \textbf{56.33} & \textbf{63.00} & \textbf{26.35} & \textbf{39.17} & \textbf{24.90}  & \textbf{19.19} \\
  &  &  & & Quarot\textsuperscript{*} & \textbf{48.75} & \textbf{57.04} & \textbf{61.99} & \textbf{26.23} & \textbf{39.14} & \textbf{24.88}  & \textbf{19.19} \\
\midrule
\multirow{6}{*}{V2-99}
  & \multirow{6}{*}{\shortstack{640\\$\times$\\1600}}
  & \multirow{6}{*}{p4}
  & \multirow{3}{*}{StreamPETR}
    & FP32     & 49.51 & 58.03 & 60.10 & 26.07 & 35.65 & 25.91  & 19.60 \\
  &  &  & & SQ       & 18.72 & 35.66 & 74.32 & 30.39 & 41.49 & 30.53  & 20.82 \\
  &  &  & & Quarot   & 19.97 & 37.49 & 71.28 & 30.11 & 40.16 & 30.23  & 20.76 \\
  &  &  & \multirow{3}{*}{\textbf{FQ-StreamPETR}}
    & FP32     & \textbf{50.48} & \textbf{58.61} & \textbf{58.78} & \textbf{26.16} & \textbf{37.05} & \textbf{25.69}  & \textbf{18.59} \\
  &  &  & & SQ\textsuperscript{*}     & \textbf{49.44} & \textbf{57.94} & \textbf{60.30} & \textbf{26.55} & \textbf{37.07} & \textbf{25.87}  & \textbf{18.59} \\
  &  &  & & Quarot\textsuperscript{*} & \textbf{50.12} & \textbf{58.39} & \textbf{58.86} & \textbf{26.16} & \textbf{37.05} & \textbf{25.87}  & \textbf{18.59} \\
\midrule
\multirow{6}{*}{V2-99}
  & \multirow{6}{*}{\shortstack{320\\$\times$\\800}}
  & \multirow{6}{*}{p4}
  & \multirow{3}{*}{PETRv2}
    & FP32     & 41.00 & 50.30 & 72.26 & 26.92 & 45.29 & 38.93  & 19.33 \\
  &  &  & & SQ       & 15.12 & 31.04 & 96.14 & 30.11 & 45.71 & 52.14  & 27.11 \\
  &  &  & & Quarot   & 19.20 & 34.15 & 95.23 & 29.76 & 45.65 & 50.46  & 26.26 \\
  &  &  & \multirow{3}{*}{\textbf{FQ-PETRv2}}
    & FP32     & \textbf{41.86} & \textbf{51.03} & \textbf{71.03} & \textbf{26.15} & \textbf{44.86} & \textbf{37.54}  & \textbf{19.04} \\
  &  &  & & SQ\textsuperscript{*}     & \textbf{40.73} & \textbf{50.61} & \textbf{71.76} & \textbf{27.02} & \textbf{45.24} & \textbf{37.97}  & \textbf{19.42} \\
  &  &  & & Quarot\textsuperscript{*} & \textbf{41.69} & \textbf{50.94} & \textbf{71.03} & \textbf{26.15} & \textbf{44.86} & \textbf{37.54}  & \textbf{19.04} \\
\bottomrule
\end{tabular}
\caption{Full nuScenes validation metrics corresponding to Table~1. SQ: SmoothQuant; Quarot: QuaRot.
A superscript “*” means QANS is applied in addition to the listed method.
\textbf{All metrics of rows with Model = FQ-\emph{x} are bolded for emphasis, consistent with the main text.}}
\label{tab:full_metrics_appendix}
\end{table*}

\section*{D \quad Full Metrics for Table~1}
\label{sec:appendix_d_full_table}

Table~\ref{tab:full_metrics_appendix} reports the complete nuScenes validation results
corresponding to Table~1, including all true–positive metrics (mATE, mASE, mAOE, mAVE, mAAE).
Unless noted, the quantization setting is full-integer INT8 with per-tensor symmetric activations.
“SQ” denotes \emph{SmoothQuant}; “Quarot” denotes \emph{QuaRot}; a superscript “*” indicates
that QANS (Quantization After Numerical Stabilization) is additionally applied.
“Bac” is the backbone; “Feat” is the feature stage used as input to the decoder.
Resolutions are listed as height~$\times$~width.
\textbf{For rows whose Model is \mbox{FQ-\emph{x}}, we bold all metrics to match the main text.}

\section{E \quad More Ablation Study}\label{sec:more_ablation}

\subsection*{E.1\quad Quantization Performance of Different Position Embeddings}
To experimentally demonstrate the superior quantization performance of our proposed QFPE, we focus solely on quantizing the positional embedding, keeping all other modules in floating-point computation. Detailed results are shown in Tab.~\ref{tab:quantization_performance_comparison_of_different_3D_position_embedding}. The original camera-ray configuration loses up to 11.97\% in mAP and 5.04\% in NDS, whereas our QFPE experiences minimal losses of only \textbf{1.42\%} in mAP and \textbf{1.15\%} in NDS.

\subsection*{E.2\quad Extension to Large Language Models}
Additionally, in large language models, the attention inputs can reach extremely large values (see Fig.~\ref{fig:llm_softmax}). We validate the effectiveness of our method in this setting as well (see Tab.~\ref{tab:comparison_qwen_deepseek}).

%% file: main.bib
@String(ICCV= {Int. Conf. Comput. Vis.})

@String(ECCV= {Eur. Conf. Comput. Vis.})

@String(ICLR = {Int. Conf. Learn. Represent.})

@String(AAAI = {AAAI})

@String(ICCV  = {ICCV})

@String(ECCV  = {ECCV})

@String(ICLR  = {ICLR})

@inproceedings{mambaquant,
title={MambaQuant: Quantizing the Mamba Family with Variance Aligned Rotation Methods},
author={Zukang Xu and Yuxuan Yue and Xing Hu and Dawei Yang and Zhihang Yuan and Zixu Jiang and Zhixuan Chen and JiangyongYu and Chen Xu and Sifan Zhou},
booktitle={The Thirteenth International Conference on Learning Representations},
year={2025},
}

@inproceedings{
xuchen2025,
title={{RWKVQ}uant: Quantizing the {RWKV} Family with Proxy Guided Hybrid of Scalar and Vector Quantization},
author={Chen Xu and Yuxuan Yue and Zukang Xu and Xing Hu and JiangyongYu and Zhixuan Chen and Sifan Zhou and Zhihang Yuan and Dawei Yang},
booktitle={Forty-second International Conference on Machine Learning},
year={2025},
url={https://openreview.net/forum?id=UOw6Qt0qYU}
}

@inproceedings{moequant,
title={Mo{EQ}uant: Enhancing Quantization for Mixture-of-Experts Large Language Models via Expert-Balanced Sampling and Affinity Guidance},
author={Zhixuan Chen and Xing Hu and Dawei Yang and Zukang Xu and Chen Xu and Zhihang Yuan and Sifan Zhou and JiangyongYu},
booktitle={Forty-second International Conference on Machine Learning},
year={2025},
url={https://openreview.net/forum?id=0epuNvt5Dj}
}

@inproceedings{point4bit,
title={Point4Bit: Post Training 4-bit Quantization for Point Cloud 3D Detection},
author={Jianyu Wang and Yu Wang and Shengjie Zhao and Sifan Zhou},
booktitle={The Thirty-ninth Annual Conference on Neural Information Processing Systems},
year={2025},
url={https://openreview.net/forum?id=sj5wiTCtu6}
}

@inproceedings{pillarhist,
  title={Pillarhist: A quantization-aware pillar feature encoder based on height-aware histogram},
  author={Zhou, Sifan and Yuan, Zhihang and Yang, Dawei and Hu, Xing and Qian, Jian and Zhao, Ziyu},
  booktitle={Proceedings of the Computer Vision and Pattern Recognition Conference},
  pages={27336--27345},
  year={2025}
}

@article{zhou2024information,
  title={Information Entropy Guided Height-aware Histogram for Quantization-friendly Pillar Feature Encoder},
  author={Zhou, Sifan and Yuan, Zhihang and Yang, Dawei and Zhao, Ziyu and Hu, Xing and Shi, Yuguang and Lu, Xiaobo and Wu, Qiang},
  journal={arXiv preprint arXiv:2405.18734},
  year={2024}
}

@article{zhou2023fastpillars,
  title={FastPillars: a deployment-friendly pillar-based 3D detector},
  author={Zhou, Sifan and Tian, Zhi and Chu, Xiangxiang and Zhang, Xinyu and Zhang, Bo and Lu, Xiaobo and Feng, Chengjian and Jie, Zequn and Chiang, Patrick Yin and Ma, Lin},
  journal={arXiv preprint arXiv:2302.02367},
  year={2023}
}

@article{liu2022petr,
  title={Petr: Position embedding transformation for multi-view 3d object detection},
  author={Liu, Yingfei and Wang, Tiancai and Zhang, Xiangyu and Sun, Jian},
  journal={arXiv preprint arXiv:2203.05625},
  year={2022}
}

@article{li2022bevformer,
  title={BEVFormer: Learning Bird's-Eye-View Representation from Multi-Camera Images via Spatiotemporal Transformers},
  author={Li, Zhiqi and Wang, Wenhai and Li, Hongyang and Xie, Enze and Sima, Chonghao and Lu, Tong and Yu, Qiao and Dai, Jifeng},
  journal={arXiv preprint arXiv:2203.17270},
  year={2022}
}

@article{second,
  title={{SECOND: Sparsely Embedded Convolutional Detection}},
  author={Yan, Yan and Mao, Yuxing and Li, Bo},
  journal={Sensors},
  volume={18},
  number={10},
  pages={3337},
  year={2018},
  publisher={Multidisciplinary Digital Publishing Institute}
}

@article{centerpoint,
  title={{Center-based 3D Object Detection and Tracking}},
  author={Yin, Tianwei and Zhou, Xingyi and Kr{\"a}henb{\"u}hl, Philipp},
  journal={arXiv preprint arXiv:2006.11275},
  year={2020}
}

@article{liu2022petrv2,
  title={Petrv2: A unified framework for 3d perception from multi-camera images},
  author={Liu, Yingfei and Yan, Junjie and Jia, Fan and Li, Shuailin and Gao, Qi and Wang, Tiancai and Zhang, Xiangyu and Sun, Jian},
  journal={arXiv preprint arXiv:2206.01256},
  year={2022}
}

@article{shu20233DPPE,
  title={3DPPE: 3D Point Positional Encoding for Multi-Camera 3D Object Detection Transformers},
  author={Shu, Changyong and Deng, Jiajun and Yu, Fisher and Liu, Yifan},
  journal={arXiv preprint arXiv:2211.14710},
  year={2023}
}

@inproceedings{lss,
  title        = {Lift, splat, shoot: Encoding images from arbitrary camera rigs by implicitly unprojecting to 3d},
  author       = {Philion, Jonah and Fidler, Sanja},
  booktitle    = {ECCV},
  year         = {2020}
}

@article{bevformer,
  title={BEVFormer: Learning Bird's-Eye-View Representation from Multi-Camera Images via Spatiotemporal Transformers},
  author={Li Zhiqi and Wang Wenhai and Li Hongyang and Xie Enze and Sima Chonghao and Lu Tong and Yu Qiao and Dai Jifeng},
  journal={arXiv preprint arXiv:2203.17270},
  year={2022}
}

@inproceedings{li2021brecq,
  title={BRECQ: Pushing the Limit of Post-Training Quantization by Block Reconstruction},
  author={Li, Yuhang and Gong, Ruihao and Tan, Xu and Yang, Yang and Hu, Peng and Zhang, Qi and Yu, Fengwei and Wang, Wei and Gu, Shi},
  booktitle={International Conference on Learning Representations},
  year={2021}
}

@inproceedings{nagel2020up,
  title={Up or down? adaptive rounding for post-training quantization},
  author={Nagel, Markus and Amjad, Rana Ali and Van Baalen, Mart and Louizos, Christos and Blankevoort, Tijmen},
  booktitle={International Conference on Machine Learning},
  pages={7197--7206},
  year={2020},
  organization={PMLR}
}

@article{bevdet4d,
  title={Bevdet4d: Exploit temporal cues in multi-camera 3d object detection},
  author={Huang Junjie and Huang Guan},
  journal={arXiv preprint arXiv:2203.17054},
  year={2022}
}

@article{bevdet,
  title={Bevdet: High-performance multi-camera 3d object detection in bird-eye-view},
  author={Huang Junjie and Huang Guan and Zhu Zheng and Du Dalong},
  journal={arXiv preprint arXiv:2112.11790},
  year={2021}
}

@article{bevdepth,
  title={BEVDepth: Acquisition of Reliable Depth for Multi-view 3D Object Detection},
  author={Li Yinhao and Ge Zheng and Yu Guanyi and Yang Jinrong and Wang Zengran and Shi Yukang and Sun Jianjian and Li Zeming},
  journal={arXiv preprint arXiv:2206.10092},
  year={2022}
}

@article{wang2022focal,
  title={Focal-PETR: Embracing Foreground for Efficient Multi-Camera 3D Object Detection},
  author={Wang, Shihao and Jiang, Xiaohui and Li, Ying},
  journal={arXiv preprint arXiv:2212.05505},
  year={2022}
}

@inproceedings{streampetr,
  title={Exploring object-centric temporal modeling for efficient multi-view 3d object detection},
  author={Wang, Shihao and Liu, Yingfei and Wang, Tiancai and Li, Ying and Zhang, Xiangyu},
  booktitle={Proceedings of the IEEE/CVF International Conference on Computer Vision},
  pages={3621--3631},
  year={2023}
}

@article{cmt,
  title={Cross Modal Transformer via Coordinates Encoding for 3D Object Dectection},
  author={Yan, Junjie and Liu, Yingfei and Sun, Jianjian and Jia, Fan and Li, Shuailin and Wang, Tiancai and Zhang, Xiangyu},
  journal={arXiv preprint arXiv:2301.01283},
  year={2023}
}

@article{hou2024open,
  title={OPEN: Object-wise Position Embedding for Multi-view 3D Object Detection},
  author={Hou, Jinghua and Wang, Tong and Ye, Xiaoqing and Liu, Zhe and Gong, Shi and Tan, Xiao and Ding, Errui and Wang, Jingdong and Bai, Xiang},
  journal={arXiv preprint arXiv:2407.10753},
  year={2024}
}

@article{wang2024omnidrive,
  title={OmniDrive: A Holistic LLM-Agent Framework for Autonomous Driving with 3D Perception, Reasoning and Planning},
  author={Wang, Shihao and Yu, Zhiding and Jiang, Xiaohui and Lan, Shiyi and Shi, Min and Chang, Nadine and Kautz, Jan and Li, Ying and Alvarez, Jose M},
  journal={arXiv preprint arXiv:2405.01533},
  year={2024}
}

@article{chu2024rayformer,
  title={RayFormer: Improving Query-Based Multi-Camera 3D Object Detection via Ray-Centric Strategies},
  author={Chu, Xiaomeng and Deng, Jiajun and You, Guoliang and Duan, Yifan and Li, Yao and Zhang, Yanyong},
  journal={arXiv preprint arXiv:2407.14923},
  year={2024}
}

@inproceedings{LQNets_eccv2018,
  author    = {Dongqing Zhang and
               Jiaolong Yang and
               Dongqiangzi Ye and
               Gang Hua},
  editor    = {Vittorio Ferrari and
               Martial Hebert and
               Cristian Sminchisescu and
               Yair Weiss},
  title     = {LQ-Nets: Learned Quantization for Highly Accurate and Compact Deep
               Neural Networks},
  booktitle = {Computer Vision - {ECCV} 2018 - 15th European Conference, Munich,
               Germany, September 8-14, 2018, Proceedings, Part {VIII}},
  series    = {Lecture Notes in Computer Science},
  volume    = {11212},
  pages     = {373--390},
  publisher = {Springer},
  year      = {2018},
  url       = {https://doi.org/10.1007/978-3-030-01237-3\_23},
  doi       = {10.1007/978-3-030-01237-3\_23},
  bibsource = {dblp computer science bibliography, https://dblp.org}
}

@inproceedings{PTQ_4bit_rapid_deployment_nips2019,
  author    = {Ron Banner and
               Yury Nahshan and
               Daniel Soudry},
  editor    = {Hanna M. Wallach and
               Hugo Larochelle and
               Alina Beygelzimer and
               Florence d'Alch{\'{e}}{-}Buc and
               Emily B. Fox and
               Roman Garnett},
  title     = {Post training 4-bit quantization of convolutional networks for rapid-deployment},
  booktitle = {Advances in Neural Information Processing Systems 32: Annual Conference
               on Neural Information Processing Systems 2019, NeurIPS 2019, December
               8-14, 2019, Vancouver, BC, Canada},
  pages     = {7948--7956},
  year      = {2019},
  url       = {https://proceedings.neurips.cc/paper/2019/hash/c0a62e133894cdce435bcb4a5df1db2d-Abstract.html},
  bibsource = {dblp computer science bibliography, https://dblp.org}
}

@inproceedings{Low_bit_quant_iccvw2019,
  author    = {Yoni Choukroun and
               Eli Kravchik and
               Fan Yang and
               Pavel Kisilev},
  title     = {Low-bit Quantization of Neural Networks for Efficient Inference},
  booktitle = {2019 {IEEE/CVF} International Conference on Computer Vision Workshops,
               {ICCV} Workshops 2019, Seoul, Korea (South), October 27-28, 2019},
  pages     = {3009--3018},
  publisher = {{IEEE}},
  year      = {2019},
  url       = {https://doi.org/10.1109/ICCVW.2019.00363},
  doi       = {10.1109/ICCVW.2019.00363},
  bibsource = {dblp computer science bibliography, https://dblp.org}
}

@article{ashkboos2024quarot,
  title={Quarot: Outlier-free 4-bit inference in rotated llms},
  author={Ashkboos, Saleh and Mohtashami, Amirkeivan and Croci, Maximilian L and Li, Bo and Cameron, Pashmina and Jaggi, Martin and Alistarh, Dan and Hoefler, Torsten and Hensman, James},
  journal={arXiv preprint arXiv:2404.00456},
  year={2024}
}

@article{hu2025ostquant,
  title={Ostquant: Refining large language model quantization with orthogonal and scaling transformations for better distribution fitting},
  author={Hu, Xing and Cheng, Yuan and Yang, Dawei and Xu, Zukang and Yuan, Zhihang and Yu, Jiangyong and Xu, Chen and Jiang, Zhe and Zhou, Sifan},
  journal={arXiv preprint arXiv:2501.13987},
  year={2025}
}

@inproceedings{yu2025mquant,
  title={Mquant: Unleashing the inference potential of multimodal large language models via static quantization},
  author={Yu, JiangYong and Zhou, Sifan and Yang, Dawei and Li, Shuoyu and Wang, Shuo and Hu, Xing and Xu, Chen and Xu, Zukang and Shu, Changyong and Yuan, Zhihang},
  booktitle={Proceedings of the 33rd ACM International Conference on Multimedia},
  pages={1783--1792},
  year={2025}
}

@article{hu2024llm,
  title={I-llm: Efficient integer-only inference for fully-quantized low-bit large language models},
  author={Hu, Xing and Cheng, Yuan and Yang, Dawei and Yuan, Zhihang and Yu, Jiangyong and Xu, Chen and Zhou, Sifan},
  journal={arXiv preprint arXiv:2405.17849},
  year={2024}
}

@inproceedings{
wang2024scantd,
title={Scan{TD}: 360{\textdegree} Scanpath Prediction based on Time-Series Diffusion},
author={Yujia Wang and Fang-Lue Zhang and Neil A. Dodgson},
booktitle={ACM Multimedia 2024},
year={2024},
url={https://openreview.net/forum?id=CGmgazTQiG}
}

@inproceedings{wang2025target,
  title={Target Scanpath-Guided 360-Degree Image Enhancement},
  author={Wang, Yujia and Zhang, Fang-Lue and Dodgson, Neil A},
  booktitle={Proceedings of the AAAI Conference on Artificial Intelligence},
  volume={39},
  pages={8169--8177},
  year={2025}
}

@article{lu2024generic,
 title={A generic layer pruning method for signal modulation recognition deep learning models},
 author={Lu, Yao and Zhu, Yutao and Li, Yuqi and Xu, Dongwei and Lin, Yun and Xuan, Qi and Yang, Xiaoniu},
 journal={IEEE Transactions on Cognitive Communications and Networking},
 year={2024},
 publisher={IEEE}
}

@article{li2025sepprune,
 title={Sepprune: Structured pruning for efficient deep speech separation},
 author={Li, Yuqi and Li, Kai and Yin, Xin and Yang, Zhifei and Dong, Junhao and Dong, Zeyu and Yang, Chuanguang and Tian, Yingli and Lu, Yao},
 journal={arXiv preprint arXiv:2505.12079},
 year={2025}
}

@inproceedings{dai2025unbiased,
  title={Unbiased Missing-modality Multimodal Learning},
  author={Dai, Ruiting and Li, Chenxi and Yan, Yandong and Mo, Lisi and Qin, Ke and He, Tao},
  booktitle={Proceedings of the IEEE/CVF International Conference on Computer Vision},
  pages={24507--24517},
  year={2025}
}

@inproceedings{yin2025knowledge,
  title={Knowledge-Aligned Counterfactual-Enhancement Diffusion Perception for Unsupervised Cross-Domain Visual Emotion Recognition},
  author={Yin, Wen and Wang, Yong and Duan, Guiduo and Zhang, Dongyang and Hu, Xin and Li, Yuan-Fang and He, Tao},
  booktitle={Proceedings of the Computer Vision and Pattern Recognition Conference},
  pages={3888--3898},
  year={2025}
}

@INPROCEEDINGS{ptq4ris,
  author={Jiang, Xiaoyan and Yang, Hang and Zhu, Kaiying and Qiu, Xihe and Zhao, Shibo and Zhou, Sifan},
  booktitle={2025 IEEE International Conference on Robotics and Automation (ICRA)}, 
  title={PTQ4RIS: Post-Training Quantization for Referring Image Segmentation}, 
  year={2025},
  pages={12663-12670},
  doi={10.1109/ICRA55743.2025.11128841}}

@article{ashkboos2024slicegpt,
  title={Slicegpt: Compress large language models by deleting rows and columns},
  author={Ashkboos, Saleh and Croci, Maximilian L and Nascimento, Marcelo Gennari do and Hoefler, Torsten and Hensman, James},
  journal={arXiv preprint arXiv:2401.15024},
  year={2024}
}

@inproceedings{zhao2024balf,
  title={Balf: Simple and efficient blur aware local feature detector},
  author={Zhao, Zhenjun},
  booktitle={Proceedings of the IEEE/CVF Winter Conference on Applications of Computer Vision},
  pages={3362--3372},
  year={2024}
}

@inproceedings{zhao2023benchmark,
  title={Benchmark for Evaluating Initialization of Visual-Inertial Odometry},
  author={Zhao, Zhenjun and Chen, Ben M},
  booktitle={2023 42nd Chinese Control Conference (CCC)},
  pages={3935--3940},
  year={2023},
  organization={IEEE}
}

@inproceedings{liu2025consistency,
  title={Consistency of local and global flatness for federated learning},
  author={Liu, Junkang and Shang, Fanhua and Tian, Yuxuan and Liu, Hongying and Liu, Yuanyuan},
  booktitle={Proceedings of the 33rd ACM International Conference on Multimedia},
  pages={3875--3883},
  year={2025}
}

@inproceedings{liu2025improving,
  title={Improving Generalization in Federated Learning with Highly Heterogeneous Data via Momentum-Based Stochastic Controlled Weight Averaging},
  author={Liu, Junkang and Liu, Yuanyuan and Shang, Fanhua and Liu, Hongying and Liu, Jin and Feng, Wei},
  booktitle={Forty-second International Conference on Machine Learning},
  year={2025}
}

@inproceedings{liu2024fedbcgd,
  title={Fedbcgd: Communication-efficient accelerated block coordinate gradient descent for federated learning},
  author={Liu, Junkang and Shang, Fanhua and Liu, Yuanyuan and Liu, Hongying and Li, Yuangang and Gong, YunXiang},
  booktitle={Proceedings of the 32nd ACM International Conference on Multimedia},
  pages={2955--2963},
  year={2024}
}

@article{tao2023dudb,
  title={DUDB: deep unfolding-based dual-branch feature fusion network for pan-sharpening remote sensing images},
  author={Tao, Hailin and Li, Jinjiang and Hua, Zhen and Zhang, Fan},
  journal={IEEE Transactions on Geoscience and Remote Sensing},
  volume={62},
  pages={1--17},
  year={2023},
  publisher={IEEE}
}

@article{zhang2023multi,
  title={Multi-scale video super-resolution transformer with polynomial approximation},
  author={Zhang, Fan and Chen, Gongguan and Wang, Hua and Li, Jinjiang and Zhang, Caiming},
  journal={IEEE Transactions on Circuits and Systems for Video Technology},
  volume={33},
  number={9},
  pages={4496--4506},
  year={2023},
  publisher={IEEE}
}

@inproceedings{xiao2023smoothquant,
  title={Smoothquant: Accurate and efficient post-training quantization for large language models},
  author={Xiao, Guangxuan and Lin, Ji and Seznec, Mickael and Wu, Hao and Demouth, Julien and Han, Song},
  booktitle={International Conference on Machine Learning},
  pages={38087--38099},
  year={2023},
  organization={PMLR}
}

@article{liu2024spinquant,
  title={SpinQuant--LLM quantization with learned rotations},
  author={Liu, Zechun and Zhao, Changsheng and Fedorov, Igor and Soran, Bilge and Choudhary, Dhruv and Krishnamoorthi, Raghuraman and Chandra, Vikas and Tian, Yuandong and Blankevoort, Tijmen},
  journal={arXiv preprint arXiv:2405.16406},
  year={2024}
}

@article{wei2022outlier,
  title={Outlier suppression: Pushing the limit of low-bit transformer language models},
  author={Wei, Xiuying and Zhang, Yunchen and Zhang, Xiangguo and Gong, Ruihao and Zhang, Shanghang and Zhang, Qi and Yu, Fengwei and Liu, Xianglong},
  journal={Advances in Neural Information Processing Systems},
  volume={35},
  pages={17402--17414},
  year={2022}
}

@article{shao2023omniquant,
  title={Omniquant: Omnidirectionally calibrated quantization for large language models},
  author={Shao, Wenqi and Chen, Mengzhao and Zhang, Zhaoyang and Xu, Peng and Zhao, Lirui and Li, Zhiqian and Zhang, Kaipeng and Gao, Peng and Qiao, Yu and Luo, Ping},
  journal={arXiv preprint arXiv:2308.13137},
  year={2023}
}

@article{esser2019learned,
  title={Learned step size quantization},
  author={Esser, Steven K and McKinstry, Jeffrey L and Bablani, Deepika and Appuswamy, Rathinakumar and Modha, Dharmendra S},
  journal={arXiv preprint arXiv:1902.08153},
  year={2019}
}

@inproceedings{bhalgat2020lsq+,
  title={Lsq+: Improving low-bit quantization through learnable offsets and better initialization},
  author={Bhalgat, Yash and Lee, Jinwon and Nagel, Markus and Blankevoort, Tijmen and Kwak, Nojun},
  booktitle={Proceedings of the IEEE/CVF conference on computer vision and pattern recognition workshops},
  pages={696--697},
  year={2020}
}

@inproceedings{nagel2019data,
  title={Data-free quantization through weight equalization and bias correction},
  author={Nagel, Markus and Baalen, Mart van and Blankevoort, Tijmen and Welling, Max},
  booktitle={Proceedings of the IEEE/CVF International Conference on Computer Vision},
  pages={1325--1334},
  year={2019}
}

@inproceedings{yuan2022ptq4vit,
  title={Ptq4vit: Post-training quantization for vision transformers with twin uniform quantization},
  author={Yuan, Zhihang and Xue, Chenhao and Chen, Yiqi and Wu, Qiang and Sun, Guangyu},
  booktitle={European conference on computer vision},
  pages={191--207},
  year={2022},
  organization={Springer}
}

@article{lin2021fq,
  title={Fq-vit: Post-training quantization for fully quantized vision transformer},
  author={Lin, Yang and Zhang, Tianyu and Sun, Peiqin and Li, Zheng and Zhou, Shuchang},
  journal={arXiv preprint arXiv:2111.13824},
  year={2021}
}

@inproceedings{li2023repq,
  title={Repq-vit: Scale reparameterization for post-training quantization of vision transformers},
  author={Li, Zhikai and Xiao, Junrui and Yang, Lianwei and Gu, Qingyi},
  booktitle={Proceedings of the IEEE/CVF International Conference on Computer Vision},
  pages={17227--17236},
  year={2023}
}

@article{shi2024p,
  title = {P\textsuperscript{2}-ViT: Power-of-Two Post-Training Quantization and Acceleration for Fully Quantized Vision Transformer},
  author={Shi, Huihong and Cheng, Xin and Mao, Wendong and Wang, Zhongfeng},
  journal={arXiv preprint arXiv:2405.19915},
  year={2024}
}

@article{yang2024post,
  title={Post-training quantization for re-parameterization via coarse \& fine weight splitting},
  author={Yang, Dawei and He, Ning and Hu, Xing and Yuan, Zhihang and Yu, Jiangyong and Xu, Chen and Jiang, Zhe},
  journal={Journal of Systems Architecture},
  volume={147},
  pages={103065},
  year={2024},
  publisher={Elsevier}
}

@inproceedings{xu2023q,
  title={Q-detr: An efficient low-bit quantized detection transformer},
  author={Xu, Sheng and Li, Yanjing and Lin, Mingbao and Gao, Peng and Guo, Guodong and L{\"u}, Jinhu and Zhang, Baochang},
  booktitle={Proceedings of the IEEE/CVF Conference on Computer Vision and Pattern Recognition},
  pages={3842--3851},
  year={2023}
}

@inproceedings{wang2024aq,
  title={AQ-DETR: Low-Bit Quantized Detection Transformer with Auxiliary Queries},
  author={Wang, Runqi and Sun, Huixin and Yang, Linlin and Lin, Shaohui and Liu, Chuanjian and Gao, Yan and Hu, Yao and Zhang, Baochang},
  booktitle={Proceedings of the AAAI Conference on Artificial Intelligence},
  volume={38},
  pages={15598--15606},
  year={2024}
}

@inproceedings{zhang2023qd,
  title={QD-BEV: quantization-aware view-guided distillation for multi-view 3D object detection},
  author={Zhang, Yifan and Dong, Zhen and Yang, Huanrui and Lu, Ming and Tseng, Cheng-Ching and Du, Yuan and Keutzer, Kurt and Du, Li and Zhang, Shanghang},
  booktitle={Proceedings of the IEEE/CVF International Conference on Computer Vision},
  pages={3825--3835},
  year={2023}
}

@article{zhou2024lidar,
  title={Lidar-PTQ: Post-training quantization for point cloud 3d object detection},
  author={Zhou, Sifan and Li, Liang and Zhang, Xinyu and Zhang, Bo and Bai, Shipeng and Sun, Miao and Zhao, Ziyu and Lu, Xiaobo and Chu, Xiangxiang},
  journal={International Conference on Learning Representations (ICLR)},
  year={2024}
}

@inproceedings{kim2021bert,
  title={I-bert: Integer-only bert quantization},
  author={Kim, Sehoon and Gholami, Amir and Yao, Zhewei and Mahoney, Michael W and Keutzer, Kurt},
  booktitle={International conference on machine learning},
  pages={5506--5518},
  year={2021},
  organization={PMLR}
}

@inproceedings{li2023vit,
  title={I-vit: Integer-only quantization for efficient vision transformer inference},
  author={Li, Zhikai and Gu, Qingyi},
  booktitle={Proceedings of the IEEE/CVF International Conference on Computer Vision},
  pages={17065--17075},
  year={2023}
}

@inproceedings{wang2018look,
  title={Look-up table unit activation function for deep convolutional neural networks},
  author={Wang, Min and Liu, Baoyuan and Foroosh, Hassan},
  booktitle={2018 IEEE Winter Conference on Applications of Computer Vision (WACV)},
  pages={1225--1233},
  year={2018},
  organization={IEEE}
}

@article{dao2023flashattention,
  title={Flashattention-2: Faster attention with better parallelism and work partitioning},
  author={Dao, Tri},
  journal={arXiv preprint arXiv:2307.08691},
  year={2023}
}

@inproceedings{yu2022nn,
  title={NN-LUT: Neural approximation of non-linear operations for efficient transformer inference},
  author={Yu, Joonsang and Park, Junki and Park, Seongmin and Kim, Minsoo and Lee, Sihwa and Lee, Dong Hyun and Choi, Jungwook},
  booktitle={Proceedings of the 59th ACM/IEEE Design Automation Conference},
  pages={577--582},
  year={2022}
}

@inproceedings{wang2023object,
  title={Object as query: Lifting any 2d object detector to 3d detection},
  author={Wang, Zitian and Huang, Zehao and Fu, Jiahui and Wang, Naiyan and Liu, Si},
  booktitle={Proceedings of the IEEE/CVF International Conference on Computer Vision},
  pages={3791--3800},
  year={2023}
}
